\documentclass[11pt]{article}
\PassOptionsToPackage{hyphens}{url}
\IfFileExists{acl.sty}{%
  \usepackage{acl}
}{%
  \usepackage[margin=1in]{geometry}
  \usepackage[round,authoryear]{natbib}
}
\usepackage[T1]{fontenc}
\usepackage[utf8]{inputenc}
\usepackage{times}
\usepackage{latexsym}
\usepackage{microtype}
\usepackage{amsmath,amssymb,amsthm,mathtools}
\usepackage{booktabs}
\usepackage{multirow}
\usepackage{graphicx}
\usepackage{xcolor}
\usepackage{tikz}
\usetikzlibrary{arrows.meta,positioning,fit,shapes.geometric}
\usepackage{url}
\usepackage{hyperref}
\hypersetup{hidelinks}
\providecommand{\pdfinfo}[1]{}
\title{Stateful CARS: Exact Cross-History Reuse for\\Policy-Constrained LLM Agents}
\author{
\textbf{Ibne Farabi Shihab}\textsuperscript{1}\thanks{Corresponding author: \texttt{ishihab@iastate.edu}}
\quad
\textbf{Md Najmus Swaqeeb}\textsuperscript{2}
\quad
\textbf{Abu Sa-Adat Mohamed Moon-Im Al Ahsan}\textsuperscript{2}
\\[4pt]
\textsuperscript{1}Department of Computer Science, Iowa State University
\\
\textsuperscript{2}Department of Computer Science \& Engineering, BRAC University
\\
\texttt{ishihab@iastate.edu},
\texttt{md.najmus.swaqeeb@g.bracu.ac.bd},\\
\texttt{abu.sa.adat.mohamed.moon.im.al.ahsan@g.bracu.ac.bd}
}

\date{}

\newtheorem{theorem}{Theorem}
\newtheorem{lemma}[theorem]{Lemma}
\newtheorem{proposition}[theorem]{Proposition}
\newtheorem{corollary}[theorem]{Corollary}
\newtheorem{definition}[theorem]{Definition}
\newtheorem{assumption}[theorem]{Assumption}

\newcommand{\actions}{\mathcal{A}}
\newcommand{\histories}{\mathcal{H}}
\newcommand{\trajectories}{\Omega}
\newcommand{\valid}{\mathcal{C}}
\newcommand{\policy}{\pi_\theta}
\newcommand{\traj}{\tau}
\newcommand{\history}{h}
\newcommand{\schemas}{\mathcal{B}}
\newcommand{\grounded}{\mathcal{W}}
\newcommand{\exclude}{\mathcal{E}}
\newcommand{\abs}{\phi}
\newcommand{\zspace}{\mathcal{Z}}
\newcommand{\residual}{r}
\newcommand{\prob}{P_\theta}
\newcommand{\cyl}{\operatorname{Cyl}}

\newcommand{\trace}{\alpha}
\newcommand{\concat}{\mathbin{\Vert}}
\newcommand{\TV}{\operatorname{TV}}
\newif\ifshowplaceholders
\showplaceholderstrue
\newcommand{\placeholder}[1]{\textcolor{red!70!black}{\textsf{[UNRUN: #1]}}}

\begin{document}
\maketitle

\begin{abstract}
Tool-using language-model agents face constraints whose meaning changes with observations and prior actions. We study exact sampling from the model distribution conditioned on a hard stateful validator while reusing invalidity certificates across histories. Stateful CARS freezes a bank of sound state--continuation schemas within each attempt and removes every trajectory containing a certified continuation at a matching abstract state. An exact residual Doob transform samples from the resulting proposal. We give a checkable future-validity bisimulation condition, prove schema soundness, adaptive exactness, i.i.d.\ outputs, almost-sure termination, monotone acceptance, and compression invariance, and characterize computation by the number of reachable full-history product states. This number can be exponential for a history-dependent language model; the evaluated method therefore makes no generic finite-trie scalability claim. On enumerable workflows, its analytic law matches the valid conditional to $10^{-16}$ at validity probability $6\times10^{-8}$, whereas state-aware local decoding can be $0.97$ away. A matched comparison is negative: observation-keyed official CARS is cheaper in sampler steps (root/Stateful ratio $0.942$ $[0.934,0.951]$), and the Qwen comparison is null ($0.99$ $[0.90,1.08]$). Cross-history transfer helps only in an internal matched-key ablation ($1.27\times$). Thus the evidence supports exact schema-induced conditioning, not a systems advantage over CARS.
\end{abstract}

\section{Introduction}

Language-model agents search databases, call APIs, update records, and run multi-step workflows, and in these settings validity is a property of an evolving trajectory rather than of an isolated string: a refund may be permitted only after authentication and an ownership check, an identifier only after a tool returned it. A locally well-formed action can therefore be globally impossible.

Let a language-model policy induce a distribution $\prob$ over complete trajectories and let $\valid$ be the set accepted by a hard validator. The faithful target is $\prob(\cdot\mid\valid)$. Rejection sampling reaches it exactly at expected cost inversely proportional to $\prob(\valid)$; local constrained decoding is cheap but stepwise masking generally yields a different distribution, because it ignores how much valid continuation mass remains after each action. The distinction matters whenever samples support self-consistency, uncertainty estimates, or any decision depending on relative probability.

Constrained Adaptive Rejection Sampling (CARS) records invalid root prefixes and subtracts their probability mass from later proposals while retaining exactness \citep{parys2025cars}. A root prefix may contain the complete serialized action--observation history, so ordinary CARS can already use a stateful validator. Its limitation is reuse: two distinct root histories that reach the same authentication or policy state occupy different trie branches. A certificate proved at one branch must be rediscovered or separately instantiated at the other. Thus, the novelty needed for stateful agents is not the ability to represent state through history; it is sound generalization across histories without changing the conditional target.

A direct state-indexed construction says that whenever abstract state $z$ is reached, a certified continuation $u$ should be excluded. Avoiding every present and future occurrence of $(z,u)$ requires the survival probability of the entire history-dependent model under these contextual exclusions. This is a Doob $h$-transform and is closely related to the future-validity quantity that makes exact global constrained decoding difficult \citep{loula2025smc,nie2026futurevalidity}. A small schema bank does not by itself make this computation small.

Stateful CARS implements this state-indexed construction directly. A bank $\schemas$ contains reusable sound schemas $(z,u)$. At the start of an attempt the bank is frozen, thereby fixing the schema-induced exclusion event for the whole draw. The sampler computes its residual mass on the reachable completion tree, using the complete serialized history in every language-model query, and applies the corresponding Doob transform. New sound schemas are committed only after the attempt terminates. This is the implementation evaluated in the paper. It is exact on any finite tree on which the residuals can be computed, but in the worst case it traverses exponentially many histories. A lazier visited-grounding variant can instead materialize only finite root-prefix exclusions and use the ordinary CARS trie; it is also exact, but we have not implemented or evaluated it and make no empirical claim for it.

We make four contributions. First, we formalize schema-induced conditioning and its freeze-and-commit update, making the evaluated exclusion event explicit. Second, we give a locally checkable future-validity bisimulation condition under which state schemas are sound. Third, we prove exactness under adaptive reuse and compression---including an unbounded sequence of i.i.d.\ outputs from $\prob(\cdot\mid\valid)$ and almost-sure termination---and state the computational boundary: exact memoization costs linear time in the reachable full-history product graph, which may be exponential in horizon unless the policy admits additional Markov structure. Fourth, we evaluate the mechanism without claiming an advantage over the closest baseline. Exactness survives validity probabilities down to $6\times10^{-8}$, but matched observation-keyed official CARS is cheaper in sampler steps, and although the abstraction obligations are now machine-verified on an observation-conditioned tool grammar (Appendix~\ref{app:verified}), the certified agent's utility claim is a null at pilot scale.

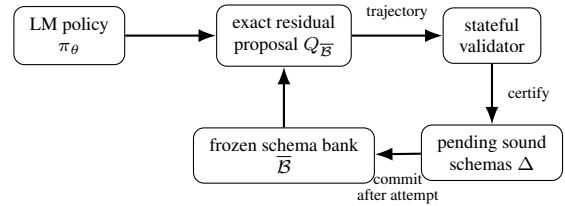
\begin{figure}[htbp]
\centering
\begin{tikzpicture}[
  node distance=0.8cm and 1.2cm, 
  box/.style={
    draw,
    rounded corners,
    align=center,
    minimum height=0.65cm,
    font=\scriptsize,
    inner xsep=6pt,
    inner ysep=4pt
  },
  arr/.style={-Latex, thick}
]

\node[box] (lm) {LM policy\\$\pi_\theta$};
\node[box, right=of lm] (proposal) {exact residual\\proposal $Q_{\overline{\mathcal{B}}}$};
\node[box, right=of proposal] (val) {stateful\\validator};

\node[box, below=0.8cm of proposal] (bank) {frozen schema bank\\$\overline{\mathcal{B}}$};
\node[box, below=0.8cm of val] (pending) {pending sound\\schemas $\Delta$};

\draw[arr] (lm) -- (proposal);

\draw[arr] (proposal) -- node[above=2pt, font=\tiny] {trajectory} (val);

\draw[arr] (bank) -- (proposal);
\draw[arr] (val) -- node[right=2pt, font=\tiny] {certify} (pending);

\draw[arr] (pending) -- node[below=3pt, font=\tiny, align=center] {commit\\[-1pt]after attempt} (bank);

\end{tikzpicture}
\caption{Evaluated Stateful CARS. A frozen schema bank induces one fixed exclusion event for the attempt. The exact residual proposal applies those schemas at every matching future state; newly certified schemas enter only after the attempt ends.}
\label{fig:stateful-cars-eval}
\end{figure}
\section{Related Work}

Grammar-, automaton-, and schema-constrained decoding reliably produces well-formed JSON or SQL. Incremental parsers reject inadmissible continuations during generation \citep{scholak2021picard,poesia2022synchromesh}, finite-state and grammar compilers make the admissible-token mask cheap enough for serving \citep{willard2023outlines,ugare2025syncode,dong2025xgrammar}, and query languages expose the constraint to the programmer \citep{beurer2023lmql}; JSONSchemaBench documents both their value and the difficulty of implementing real schemas faithfully \citep{geng2025jsonschema}. All of these mask locally inadmissible tokens and renormalize, which is exact only when the surviving actions carry equal future-validity mass. Otherwise the terminal distribution shifts, a gap that grammar-aligned decoding attacks by reweighting toward the true conditional \citep{park2024gad} and that recent analyses of future validity make explicit \citep{nie2026futurevalidity}. Our setting differs in what the constraint depends on: these systems condition on the token prefix under a fixed grammar, whereas a stateful validator's admissibility changes with tool observations, so the same continuation is legal or illegal depending on what the environment returned.

Rejection sampling is globally exact but can be prohibitively expensive. CARS stores invalid prefixes adaptively and reweights the model to avoid their finite union, staying exact with monotone acceptance \citep{parys2025cars}; adaptive weighted rejection sampling accelerates black-box constraint evaluation with unbiased normalizer estimates \citep{lipkin2025awrs}, and sequential Monte Carlo gives flexible approximate inference under syntactic and semantic constraints \citep{loula2025smc}, with twisted variants learning proposals that anticipate the constraint \citep{zhao2024tsmc,kim2025sdtsmc}. The difficulty those methods confront is the same one we do: the future-validity statistic is what a locally masked proposal lacks, and estimating rather than ignoring it is the active line of work \citep{nie2026future}. Stateful CARS remains in the exact adaptive-rejection family but changes the exclusion object: a sound schema applies at every matching state rather than at one concrete root prefix. This generalization requires an exact contextual residual and can be exponentially expensive for a history-dependent model. Our contribution is therefore the sound state abstraction, the exact adaptive schema-induced construction, and its complexity characterization, not a claim that schema reuse retains CARS's finite-trie cost.

Tool-use research studies action selection over large API collections \citep{qin2023toolllm,patil2023gorilla}, stable simulation \citep{guo2024stabletoolbench}, dynamic user--agent interaction \citep{yao2024taubench}, reasoning-and-acting \citep{yao2023react,schick2023toolformer}, and feedback-driven repair \citep{madaan2023selfrefine}, and it is evaluated on broad agentic suites spanning interactive environments, web tasks, and repository-level software engineering \citep{liu2024agentbench,zhou2024webarena,jimenez2024swebench}. Those benchmarks measure task success; none defines an exact conditional trajectory law, which is the object we sample from. Their action spaces are also far larger than the bounded grammar we can enumerate, so they are the natural target for the residual construction and simultaneously the regime where its cost boundary binds hardest.

Formal guardrails supply the kind of certificates required here. Agent-C compiles temporal policies and blocks non-compliant actions \citep{kamath2025agentc}; ToolGate tracks symbolic state and verifies tool preconditions and postconditions \citep{toolgate2026}, reusing the precondition/postcondition discipline that axiomatic program semantics introduced \citep{hoare1969axiomatic}, and such checks are typically discharged by an SMT solver \citep{demoura2008z3}. Those systems answer whether an action is allowed. Stateful CARS addresses a complementary question: how to reuse such answers across histories while preserving the language model's relative probability over all trajectories that the validator accepts. Our abstraction obligations are of exactly the form a solver could discharge, and Appendix~\ref{app:verified} discharges them by exhaustive enumeration on a bounded grammar rather than symbolically, which is sufficient there and would not scale to an unbounded one.

\section{Problem Formulation}
\label{sec:setup}

We use a finite action alphabet $\actions$, finite horizon $H$, and a deterministic environment: from $s_t,a_t$ it produces the next state and observation, both appended to the model history $\history_{t+1}$ (illegal actions move to a terminal failure state, so the transition is total). The policy may depend on the whole serialized history, $a_t\sim\policy(\cdot\mid\history_t)$, and together with the environment induces a measure $\prob$ over complete trajectories $\trajectories$; Appendix~\ref{app:stochastic} treats stochastic tool outcomes. A deterministic validator defines $\valid\subseteq\trajectories$ with $\prob(\valid)>0$, and the target is
\begin{equation}
  \prob(\traj\mid\valid)=\frac{\prob(\traj)\mathbf{1}\{\traj\in\valid\}}{\prob(\valid)}.
  \label{eq:target}
\end{equation}
Each reachable history $\history$ has a unique \emph{root trace} $\trace(\history)$: the complete serialized prefix---initial context, actions, and all observations so far---not the action sequence alone. This matters because when validity depends on an observation, two histories with identical actions but different observations admit different valid completions, so an action-only key would not identify the branch and Lemma~\ref{lem:grounding} would fail. The probability space is therefore over action--observation trajectories and $\prob$ in Equation~\eqref{eq:target} is the joint law. Our environment draws a hidden ownership world $c\in\{\textsc{own},\textsc{other}\}$ once per episode from a prior $\rho$ and reveals it by \texttt{PROBE}, so $\prob(\valid)=\sum_c\rho(c)\prob(\valid\mid c)$ and the target is the corresponding $\rho$-mixture over the enumerated $10{,}922$ outcomes; Appendix~\ref{app:stochastic} gives the kernel, the world-sampling protocol, and the two-satisfiable-world validation. Let $\cyl(w)$ be the complete trajectories whose root trace begins with $w$; $w$ is invalid when $\cyl(w)\cap\valid=\emptyset$. Write $u\in L(\history)$ when some suffix completes $u$ from $\history$ into $\valid$; this future-valid continuation language depends on policy state, observations, bindings, and remaining budget, but never on the probability the model assigns.

\begin{definition}[Validity-preserving abstraction]
An abstraction $\abs:\histories\to\zspace$ is validity-preserving when
\[
  \abs(\history)=\abs(\history')\quad\Longrightarrow\quad
  L(\history)=L(\history').
\]
The remaining action budget is included in $\abs$ whenever validity depends on it.
\end{definition}

The definition states the exact requirement but is not, by itself, a construction. The following local condition is checkable for finite symbolic validators. Let $F(z,a)$ be an abstract successor, let $b(z)$ identify terminal abstract states, and let $c(z)$ give terminal acceptance.

\begin{assumption}[Future-validity bisimulation]
For any reachable $\history,\history'$ with $\abs(\history)=\abs(\history')=z$: (i) both are terminal or both are nonterminal; (ii) if terminal, they have the same acceptance value $c(z)$; and (iii) if nonterminal, then for every $a\in\actions$,
\[
  \abs(\operatorname{Succ}(\history,a))
  =F(z,a)
  =\abs(\operatorname{Succ}(\history',a)).
\]
\end{assumption}

\begin{proposition}[Local condition implies validity preservation]
Under the finite horizon and future-validity bisimulation assumption, $\abs$ is validity-preserving.
\label{prop:bisim}
\end{proposition}

The proof is a backward induction on remaining horizon (Appendix~\ref{app:bisimulation}). In practice, $z$ can be the product of the validator-automaton state, relevant database predicates, bound identifiers, permission state, and remaining budget. A full history is always conservative; the coarsest possible abstraction is the right congruence induced by equality of $L(\history)$.

\begin{definition}[Sound schema and grounding]
A state-continuation schema $(z,u)$ is sound if $u\notin L(\history)$ for every reachable $\history$ with $\abs(\history)=z$. Its grounding at such a history is
\[
  g(\history;z,u)=\trace(\history)\concat u.
\]
\end{definition}

For a bank $\schemas$ of schemas, define its \emph{schema-induced exclusion event}
\begin{equation}
\begin{aligned}
\exclude_{\schemas}
=\bigl\{\traj\in\trajectories:\;&
\exists\, t,\,(z,u)\in\schemas,\\
&\abs(\history_t(\traj))=z,\quad
a_{t:t+\lvert u\rvert-1}(\traj)=u
\bigr\}.
\end{aligned}
\label{eq:schema-event}
\end{equation}
Thus a schema applies at every matching history, including histories not previously visited by the sampler. Every occurrence of a sound schema is non-completable, so $\exclude_{\schemas}\cap\valid=\emptyset$. A concrete grounding remains useful for the separate visited-grounding variant, but it is not the exclusion object used by the evaluated implementation.

\section{Stateful CARS}
\label{sec:method}

The evaluated algorithm maintains a bank $\schemas$ of sound schemas. For a reachable full history $\history$, define the residual mass of the schema-induced survivor event by
\begin{equation}
  \residual_{\schemas}(\history)
  =\prob(\exclude_{\schemas}^{c}\mid\history).
  \label{eq:residual}
\end{equation}
When $\residual_{\schemas}(\history)>0$, the proposal is
\begin{equation}
 Q_{\schemas}(a\mid\history)
 =\frac{\policy(a\mid\history)
 \residual_{\schemas}(\operatorname{Succ}(\history,a))}
 {\residual_{\schemas}(\history)}.
 \label{eq:proposal}
\end{equation}
The denominator equals the sum of the numerators by the law of total probability.

At the beginning of attempt $k$, the algorithm freezes $\overline{\schemas}_k$. It then evaluates Equations~\eqref{eq:residual}--\eqref{eq:proposal} for that fixed bank and never writes proposal memory before the attempt terminates. The final validator accepts or rejects the trajectory. At the boundary, a certificate procedure may add a pending set $\Delta_k$ of schemas, but each entry must be proved non-completable for \emph{every} reachable history in its abstract class. The official update examines invalid siblings at every visited proper prefix after both valid and invalid attempts; a realized failure in one hidden world is insufficient unless the abstraction-level proof covers every world the abstraction leaves open. The next bank is $\schemas_{k+1}=\kappa(\schemas_k\cup\Delta_k)$, where $\kappa$ preserves the induced exclusion event. Algorithm~\ref{alg:stateful-cars} gives the complete procedure.

The complete serialized history remains the policy key. The abstraction determines where a certificate applies, but it never licenses replacing $\policy(a\mid\history)$ by $\policy(a\mid\abs(\history))$. In a deterministic environment, a schema matcher can track the active suffixes of patterns in $\schemas$, but the residual cache must still distinguish any histories on which the language model may assign different action probabilities.

\subsection{Computational boundary}

Let $\mathcal{R}_{\schemas}$ be the set of reachable product nodes explored by the exact residual recursion, where a node contains the complete serialized history and the deterministic schema-matcher state. Memoized backward recursion visits each such node once.

\begin{proposition}[Residual complexity]
For a fixed schema bank in a deterministic finite-horizon environment, exact schema-induced residuals require $O(|\actions|\,|\mathcal{R}_{\schemas}|)$ arithmetic operations, $O(|\mathcal{R}_{\schemas}|)$ cached values, and at most one policy-vector evaluation per distinct full history. In the worst case $|\mathcal{R}_{\schemas}|=\Theta(|\actions|^H)$. If the policy and environment admit a finite Markov state $x(\history)$ and the schema matcher has state space $M_{\schemas}$, the recursion can instead be memoized on $(t,x,m)$ in $O(H|X||M_{\schemas}||\actions|)$ time.
\label{prop:complexity}
\end{proposition}

The upper bounds follow by counting product nodes and outgoing actions. The exponential lower bound is attained when every action prefix induces a distinct policy state and all branches survive until depth $H$. Appendix~\ref{app:complexity} gives the proof and the exact cache-key requirement. This boundary is central: the evaluated implementation is practical in our enumerable workflows, but a small abstraction or schema bank alone does not make it scalable to a free-form history-dependent agent.

\subsection{A conservative visited-grounding variant}

One can avoid applying a schema at unvisited future occurrences by maintaining a finite root-prefix set $\grounded$. When a matching history is actually visited, $g(\history;z,u)$ is placed in a pending buffer and committed after the attempt. The next proposal conditions only on
\[
\exclude_{\grounded}=\bigcup_{w\in\grounded}\cyl(w).
\]
This is ordinary finite-trie CARS over sound grounded prefixes. It can miss reusable exclusions until their histories have been visited, but the fixed-exclusion proof below still makes every returned sample exact. We include it to separate a possible scalable engineering route from the evaluated method. Appendix~\ref{app:vg} implements and measures it: its obligations hold and its steady-state cost matches both comparators on three domains.

\subsection{Lossless compression}

A schema-bank compression $\kappa$ is lossless when
\[
  \exclude_{\kappa(\schemas)}=\exclude_{\schemas}.
\]
Within one abstract state, $(z,u\concat v)$ is dominated by $(z,u)$ because every occurrence of the shorter pattern already excludes the longer one. A trie, deterministic matcher, Patricia compression, or shared immutable subgraph is valid only when it recognizes exactly the same event in Equation~\eqref{eq:schema-event}. Merging abstract states requires the bisimulation obligation in Proposition~\ref{prop:bisim}. For the visited-grounding variant, ordinary root-prefix antichain compression is lossless under the analogous equality $\exclude_{\kappa(\grounded)}=\exclude_{\grounded}$.

\section{Correctness}
\label{sec:correctness}

The guarantees require finite horizon and action space, exact full-history policy probabilities for every explored node, sound schemas, event-preserving compression, and an exclusion event that is frozen within each attempt. The certificate procedure may be incomplete: a missed certificate reduces efficiency but not validity.

\begin{lemma}[Fixed-exclusion proposal]
Let $E\subseteq\trajectories$ be fixed with $\prob(E^c)>0$, and define $r_E(\history)=\prob(E^c\mid\history)$. Sequential sampling with
\[
Q_E(a\mid\history)
=\frac{\policy(a\mid\history)r_E(\operatorname{Succ}(\history,a))}
{r_E(\history)}
\]
produces
\[
  Q_E(\traj)=\prob(\traj\mid E^c).
\]
If $E\cap\valid=\emptyset$, then $Q_E(\cdot\mid\valid)=\prob(\cdot\mid\valid)$.
\label{lem:fixed}
\end{lemma}

The proof telescopes residual ratios along a complete trajectory and does not require $E$ to be a trie union. It therefore covers both $E=\exclude_{\schemas}$ and $E=\exclude_{\grounded}$; Appendix~\ref{app:fixed} gives the complete argument.

\begin{lemma}[Schema-induced soundness]
If every member of $\schemas$ is sound, then $\exclude_{\schemas}\cap\valid=\emptyset$.
\label{lem:schema-soundness}
\end{lemma}

Indeed, a valid trajectory containing an occurrence of $(z,u)$ would make $u\in L(\history_t)$ at a history with abstraction $z$, contradicting schema soundness. This direct event-level lemma is the link between the evaluated implementation and Lemma~\ref{lem:fixed}.

\begin{theorem}[Adaptive exactness, termination, and repeated samples]
Assume $\prob(\valid)>0$ and the conditions above. Before each attempt, an arbitrary history-dependent update may select a sound bank $\schemas_k$, provided the bank remains fixed throughout that attempt. For every returned trajectory $Y_j$ and the sigma-field $\mathcal{F}_{j-1}$ generated by all earlier attempts and returned trajectories,
\[
  \Pr(Y_j\in B\mid\mathcal{F}_{j-1})
  =\prob(B\mid\valid)
\]
for every $B\subseteq\trajectories$. Hence all returned trajectories are i.i.d. from $\prob(\cdot\mid\valid)$. Each return occurs almost surely, and the conditional expected number of attempts before the next return is at most $1/\prob(\valid)$.
\label{thm:adaptive}
\end{theorem}

The conditioning happens at an attempt boundary: given the complete past, the next bank and hence $\exclude_{\schemas_k}$ are fixed and sound. Lemmas~\ref{lem:schema-soundness}--\ref{lem:fixed} make the accepted law independent of the realized bank, and $Q_{\schemas_k}(\valid)=\prob(\valid)/\prob(\exclude_{\schemas_k}^{c})\geq\prob(\valid)$ gives the geometric tail bound. Appendix~\ref{app:adaptive} treats random attempt indices, pending updates, and sound early stopping formally.

\begin{proposition}[Monotone acceptance]
If fixed exclusion events $E\subseteq E'$ are both disjoint from $\valid$, then
\[
  Q_{E'}(\valid)\geq Q_E(\valid).
\]
\label{prop:monotone}
\end{proposition}

This is a proposal-attempt guarantee, not a wall-clock guarantee. Schema matching, full-tree residual evaluation, language-model scoring, validation, and memory updates must all be included in end-to-end accounting.

\begin{proposition}[Exact cost of a false exclusion]
For an arbitrary exclusion set $E$ with $\prob(\valid\setminus E)>0$, sampling from $\prob(\cdot\mid E^c)$ and accepting $\valid$ returns $\prob(\cdot\mid\valid\setminus E)$, and
\[
  \TV\!\left(\prob(\cdot\mid\valid),
  \prob(\cdot\mid\valid\setminus E)\right)
  =\prob(E\mid\valid).
\]
\label{prop:false-exclusion}
\end{proposition}

Thus an unsound abstraction is not a benign approximation: the removed fraction of valid mass is exactly the total-variation error. This identity motivates the exhaustive abstraction audit used in the controlled experiment and the false-exclusion measurements required on larger tasks.

\section{Controlled Evaluation}
\label{sec:experiments}

The experiments ask whether the evaluated sampler reproduces the exact conditional when validity is rare enough to starve rejection sampling and bias local masking, whether the implementation really freezes its bank and realizes the exclusion event of Equation~\eqref{eq:schema-event}, what it costs once the residual recursion is counted, and whether cross-history transfer beats a fully matched official baseline. The answers are positive on exactness and negative on efficiency, and we report both at the same weight.

\subsection{Enumerable stateful workflow}

The environment is a bounded customer-service workflow over \texttt{AUTH}, \texttt{PROBE}, \texttt{REFUND}, \texttt{READ}, and \texttt{STOP}, with six action slots (a trajectory still running after six actions is truncated and invalid), giving $5461$ complete action traces per hidden world. We state the slot count because it fixes the trace space. A refund requires authentication, at least one read, a probe whose observation confirms ownership, and exactly one success, so the same continuation is valid or invalid depending on the observed ownership; this is why the root key must carry observations. The support holds $60$ valid trajectories at $\prob(\valid)\approx3.0\times10^{-3}$, and the target is computed by enumerating the complete action tree. The abstraction records authentication, probe and read status, observed ownership, refund count, and remaining depth, and exhaustive enumeration verifies the bisimulation condition and confirms that every inspected frozen bank excludes zero valid trajectories. One accounting detail matters because the formal setup and the legacy baseline use different keys: an observation-keyed root memory can condition on the ownership reveal, so no both-world-dead argument constrains it, whereas the action-keyed baseline of the earlier diagnostic must exclude only the $5401$ traces dead under both worlds and abstain on the $192$ observation-grounded continuations. The decisive baseline later is observation-keyed and carries no such handicap. Each attempt draws the hidden world $c\sim\rho$ afresh and a failure discards the whole attempt including $c$, which is exact here by Corollary~\ref{cor:oneworld} because one world is satisfiable, and \emph{not} by a general joint-space argument, which would need the residual-reweighted prior (Appendix~\ref{app:worldsampling}).

Sampling diagnostics behave as the theorem requires without themselves proving it. A persistent bank at $N=6000$ gives TV $0.033$ with goodness-of-fit $p=0.23$, a fresh bank gives $0.087$ and $p=0.56$, and rejection sampling gives $0.028$ and $p=0.05$. Across a rarity sweep, rejection calls per valid output climb from $1.5\times10^3$ to $1.9\times10^5$ while the pilot stays near $6.6$ sampler steps once reusable failure regions are learned, for step ratios of $234\times$, $662\times$, $3.7\times10^3$, and $2.9\times10^4$ against terminal rejection alone. Per-attempt acceptance rises monotonically from $3.0\times10^{-3}$ to $1.0\times10^{-2}$ over 24 insertions, and canonicalizing $1707$ raw certificates to $569$ entries leaves every enumerated residual mass unchanged at machine precision, supporting Proposition~\ref{prop:monotone} and compression invariance.

\subsection{Conformance, abstraction, and reuse}
\label{sec:conformance}

An implementation audit checks the obligations of the algorithm actually run (Table~\ref{tab:conformance}). Over $300$ instrumented attempts the bank performs zero writes between the first proposal action and final validation, with every update in the boundary caller, and enumeration confirms $\exclude_{\schemas}\cap\valid=\emptyset$ for every frozen bank inspected. One bank of $193$ schemas removes $11{,}820$ of $13{,}650$ continuation \emph{edges} ($1365$ prefixes $\times$ five actions $\times$ two worlds, not the $10{,}922$ outcomes) while retaining all valid mass, and enumerating that frozen proposal gives analytic TV $1.5\times10^{-16}$ over the full $60$-trajectory support. The abstraction ablation (Table~\ref{tab:ablation}) forms a verified refinement chain across the $572$ reachable state and depth pairs in which all three sound maps reach TV $\approx10^{-16}$, whereas dropping observed ownership removes the entire valid mass, an endpoint Proposition~\ref{prop:false-exclusion} cannot characterize because no survivor law remains. Our map is conservative rather than minimal, since the coarsest sound partition has $21$ classes. A simultaneous multinomial band adds a distribution-free view: at $N=8000$ the observed TV is $0.0215$, inside the $95\%$ null band $[0.0185,0.0314]$, with the maximum per-cell deviation within its null quantile.

\label{sec:local-fair}
Local decoding's weakness is not blindness to state. The strongest local decoder this environment admits reads the observed ownership, masks every immediately inadmissible or provably doomed action, and needs no rejection step, yet its enumerated limiting law sits at $\TV=0.742$ from the target while the identical enumeration weighted by residual mass gives $5.8\times10^{-17}$, so local renormalization alone is the entire source of the bias (Appendix~\ref{app:fairlocal}). The abstraction is also what makes reuse \emph{sound}, which a literal action-prefix cache cannot be under a stateful validator: an eager literal cache falsely excludes the whole valid mass, and a conservative one must abstain on all $192$ observation-grounded continuations that the abstraction certifies.

Two reuse results survive and they are narrow. Isolating the certificate's applicability key, so that $\abs(\history)$ and $(\abs(\history),\trace(\history))$ differ in nothing but transfer while sharing the information space, update rule, and byte-identical derivation code, transfer cuts sampler steps per accept from $13.90\pm0.20$ to $10.98\pm0.11$ over five paired seeds at $N=1500$, a paired difference of $2.92$ $[2.66,3.18]$ or $1.27\times$ $[1.24,1.29]$, with both arms excluding zero valid mass (Appendix~\ref{app:reuseabl}). Separately, holding the excluded set identical by construction, $45$ abstraction-indexed certificates or $61$ grounded prefixes describe it, because $17$ certificate-bearing abstract states are reached by $133$ histories at mean fan-out $7.8$ and maximum $54$; a canonical-order control gives $1.0\times$ at fan-out one (Table~\ref{tab:reuse}). These are commensurable symbolic counts rather than bytes, runtime, or run-dependent memory, and we withdraw the separate raw-memory comparison because its two memories induce different exclusion sets (Appendix~\ref{app:obskeyed}).

The decisive comparison reverses our original claim. That claim pitted a post-invalid-only Stateful ablation against a memory keyed on invalid \emph{action} prefixes, so the root arm saw less than the abstraction and the Stateful arm updated more weakly. Removing both confounds, by implementing official root-prefix CARS over the complete action and observation trace and giving Stateful CARS the same official update after valid and invalid attempts \citep{parys2025cars}, matching information alone gives root/Stateful $0.838$ $[0.833,0.843]$ and matching the update as Algorithm~\ref{alg:stateful-cars} specifies gives $0.942$ $[0.934,0.951]$, still favoring root-prefix CARS with an interval excluding one (Table~\ref{tab:matched-update}). We report $0.942$ as the fair result and retract the earlier step-count advantage along with the run-dependent memory-economy claim.

Since exactness alone distinguishes nothing, three adjacent samplers run under the identical protocol (Table~\ref{tab:adjacent}). Terminal rejection is exact but pays $1557$ calls per accept against $8.3$ for the post-invalid ablation and $7.36$ for official-update Stateful; independence Metropolis is cheap yet returns $99.5\%$ duplicates at $\TV=0.707$; and the $K=64$ particle filter retains $\TV=0.081$. That filter is \emph{consistent} rather than biased, because its mask-normalizer weights correct the proposal, so its finite-$K$ error is self-normalized approximation and its cost is per \emph{independent} sample (Appendix~\ref{app:particle}). Token-level AWRS runs separately (Appendix~\ref{app:incremental}).

Replacing the synthetic policy with restricted action-label probabilities from Qwen2.5-\{7B,1.5B,0.5B\}-Instruct \citep{qwen2024qwen25} covers all fifteen model and temperature cells (Appendix~\ref{app:llm-results}). On the eight genuinely rare cells ($\prob(\valid)\le10^{-3}$) the strongest state-aware local decoder sits at TV $0.359$ to $0.967$ from the exact conditional while Stateful CARS stays exact, at TV $0.018$ to $0.102$ from finite samples and $3.5\times10^{-17}$ to $1.5\times10^{-16}$ analytically on the five rarest. The sampler-step ratio against \emph{terminal rejection} runs from $605\times$ to $8.5\times10^{6}$ and grows as validity thins, which is not a comparison with CARS or AWRS-SMC. The 7B cells say nothing about fidelity, since at $T\le1.5$ a single trajectory carries all but ${\le}5\times10^{-5}$ of the target mass and at $T=2.0$ the empirical difference is near sampling resolution, with a terminal-rejection ratio of only $1.6\times$.

\subsection{Efficiency accounting}
\label{sec:efficiency-accounting}

Sampler steps are not the whole cost, because the residual recursion issues policy queries of its own. Instrumenting the 1.5B, $T=1.0$ cell ($\prob(\valid)=1.7\times10^{-4}$), each accepted sample costs $10.6$ sampler steps at $28.2$ residual invocations per step, so $299$ invocations per accept against $2.5\times10^{4}$ for rejection sampling, an $84\times$ reduction that is smaller than the step ratio precisely because the global recursion is counted. Distinct forward passes give the counterpoint: Stateful CARS issues \emph{more} of them than rejection sampling ($1365$ against $392$ at $N=240$, $33$\,s against $10$\,s of model time) because the recursion scores the reachable tree rather than only sampled paths, though end-to-end it still draws $240$ valid samples in $38$\,s against $143$\,s, a $3.8\times$ win. In this experiment the policy depends only on the action prefix, so that prefix is a sound cache key; an observation-conditioned policy would require the complete serialized history.

The matched B and C profile is measured and genuinely mixed (Table~\ref{tab:bc-matched}). Arm C is cheaper in sampler steps, the metric we lead with, while arm B is cheaper in wall-clock ($1.268\times$, $[1.155,1.381]$) and, at matched pruning power, $6.10\times$ smaller in serialized bytes. Under a real served policy on an L4 the pattern sharpens (Appendix~\ref{app:served}): arm C issues $30\%$ fewer sampler decisions but $1.3\%$ \emph{more} forward passes, so GPU-seconds and throughput tie and the two exact arms cost the same on an accelerator. We therefore make no runtime claim in either direction, the honest difference being memory; batched serving is unmeasured, and the earlier partially matched run favoring root-prefix at $0.42\times$ cold is a diagnostic only. A synthetic-policy ratio also need not survive a real policy, so the two certificate keys are compared under the LM paired at the seed level, with both arms reading the same memoized Qwen table and sharing the trajectory RNG within a seed. Pooled, the ratio is $0.99$ $[0.90,1.08]$ at sign-test $p=0.61$ and the per-cell sign reverses, with identical distinct-forward counts ($1365$) and zero excluded valid mass. Fan-out explains why transfer helps in the key-only ablation but establishes no advantage over matched CARS, which wins even in this high-fan-out workflow, so we make no comparative efficiency claim. Exactness is untouched, holding to $10^{-16}$ in the five analytically enumerated rare cells.

Three structurally different observation-grounded validators, a refund workflow, an SQL transaction, and a booking API, each with its own policy and enumerated target at a matched $2500$-attempt budget, confirm the mechanism is not tuned to one workflow. Stateful CARS is exact in all three, at zero excluded valid mass and TV $0.025$ to $0.105$, returning $799$ to $1190$ valid samples where rejection sampling returns $0$ to $4$ and local masking stays biased. The fixed-prune certificate ratios of $1.36\times$, $1.27\times$, and $1.00\times$ are representation diagnostics rather than byte-memory or runtime claims (Table~\ref{tab:multidomain}).

\subsection{Live-system validators}
\label{sec:live-systems}

Two validators can be swapped for their real backends while keeping the bounded grammar. Against a live \texttt{sqlite3} database ($P(\valid)=2.6\times10^{-5}$) and the released \texttt{tau-bench} retail tools at a pinned commit, Stateful CARS reproduces the exact conditional at TV $0.002$ and $0.040$ with zero excluded valid mass, where rejection sampling never accepts on the rare SQL task and local masking stays biased (Table~\ref{tab:live}). A live-agent probe on a 15-task $\tau$-bench retail subset exercises only the reuse component and establishes \emph{cache} reuse rather than certified sound reuse, with every outcome comparison inconclusive at $n=15$ (Appendix~\ref{app:liveagentmain}).
\section{Conclusion}

Stateful constraints do not by themselves defeat root-prefix CARS, since a complete action--observation history already identifies a concrete branch. The object studied here is instead a sound schema that applies across histories. Stateful CARS freezes a schema bank, makes the induced exclusion event explicit, and applies its exact residual transform on the reachable completion tree, yielding trajectories that are adaptively i.i.d.\ from the valid conditional---including under stochastic outcomes, once the outcome draw is residual-reweighted. The construction's cost boundary is measured, not merely asserted: reachable nodes grow with fitted base $2.99$ against $|\actions|=3$ under a full-history policy, and linearly in the horizon when the policy factors through a finite Markov state (Appendix~\ref{app:scaling}). Our workflows are feasible because their action trees are enumerable. The efficiency verdict is negative and we state it plainly. Under matched observations and updates, root-prefix CARS is cheaper in sampler steps ($0.942$ $[0.934,0.951]$); the Qwen key comparison is null; the live-agent probe is uncertified. Systems measurements split rather than rescue this: Stateful CARS wins wall-clock and, at matched pruning power, $6.10\times$ in serialized bytes, while under a served model the baseline's step advantage inverts into a forward-pass deficit and GPU-seconds tie (Appendices~\ref{app:matched-systems}, \ref{app:served}). What survives is an exact construction with a measured complexity boundary and a memory advantage---not a speed claim.

\section{Limitations}

Exactness applies to hard validators, not to subjective notions such as helpfulness unless they are converted into a deterministic acceptance predicate. The method faithfully conditions on the validator it is given; an incorrect policy specification remains incorrect after exact sampling. Sound state reuse requires either a proved bisimulation or a conservative abstraction. Proposition~\ref{prop:false-exclusion} shows the precise cost of violating this requirement.

The evaluated proposal computes a future-survival residual and can be exponential in horizon. Its cache is safe only when keyed by the complete serialized history together with deterministic matcher state; memoizing merely by abstract state would incorrectly merge language-model conditionals. The visited-grounding finite-trie variant avoids global application of schemas, but it has not been implemented or evaluated here.

The reuse advantage over root-prefix CARS does not survive a matched information space and update. The official baseline is exact and needs fewer sampler steps than Stateful CARS (paired root/Stateful ratio $0.942$ $[0.934,0.951]$), so we claim no step-count, wall-clock, or storage advantage over the closest exact method. The earlier $1.36\times$ result used an action-keyed baseline and a weaker Stateful update. The fixed-prune $45$-versus-$61$ certificate count is retained only because the represented exclusion set is identical; the run-dependent raw-memory comparison is withdrawn. Byte storage and the matched B--C runtime are measured (Appendices~\ref{app:matched-systems}, \ref{app:served}): at matched pruning Stateful CARS is $6.10\times$ smaller in bytes and the arms tie on GPU-seconds. A deployed observation-conditioned agent with certified high fan-out remains unmeasured.

Our evidence covers three enumerable workflows, two real execution backends, and one live-agent dialogue, all under a bounded action grammar. Exactness is established only where the target is enumerable, since verifying it requires the exact conditional; the rarity regime in which we verify it ($\prob(\valid)$ down to $6\times10^{-8}$) is genuinely adversarial for the alternatives, but it is still an enumerable environment, and we do not claim verified exactness on a free-form task. On the free-form dialogue we tested only cross-history reuse, so we claim no end-to-end reward gain there (\S\ref{sec:live-agent}), and we have no results on temporal-policy suites or large action spaces. Turning sound invalid-call avoidance into a measurable live-agent success gain, at a scale where reward is not noise-dominated, is open. We also could not run Agent-C or ToolGate as baselines: the Agent-C repository is public but holds only a README promising future release, with an unanswered code-release issue, and ToolGate resolves to two contemporaneous papers by different groups, neither linking an implementation. The rollout-call metric does not by itself expose GPU time, batching, or validator cost; \S\ref{sec:efficiency-accounting} reports measured forward passes, wall-clock, and a break-even sample count.

The development assumes deterministic tool outcomes and bounded trajectories; Appendix~\ref{app:stochastic} gives the stochastic modification, under which the sound-schema condition strengthens to excluding valid completions on every positive-probability outcome path. Approximate residuals, probabilistic membership filters, and approximate state merging fall outside the exact guarantee without an added correction.

\bibliography{references_eacl_updated}

\appendix

\section{Algorithm}
\label{app:alg}

\begin{figure}[t]
\centering
\small
\begin{minipage}{0.96\columnwidth}
\hrule
\vspace{2pt}
\textsc{Stateful CARS}$(\policy,\textsc{Validate},\abs)$
\vspace{2pt}
\hrule
\begin{enumerate}
\setlength{\itemsep}{1pt}
\setlength{\parskip}{0pt}
\item Initialize the schema bank $\schemas\leftarrow\emptyset$.
\item At the start of attempt $k$, freeze $\overline{\schemas}\leftarrow\schemas$ and initialize a pending schema set $\Delta\leftarrow\emptyset$.
\item Construct the fixed event $\exclude_{\overline{\schemas}}$ from Equation~\eqref{eq:schema-event}. Compute $r_{\overline{\schemas}}$ by memoized recursion keyed by the complete serialized history and matcher state, and sample the attempt from $Q_{\overline{\schemas}}$ in Equation~\eqref{eq:proposal}. Do not mutate $\overline{\schemas}$ or the residual cache's defining event.
\item If a sound monitor proves that the realized prefix has no valid completion, terminate the attempt as invalid; otherwise continue to a terminal trajectory and apply \textsc{Validate}.
\item After the attempt ends, inspect every visited proper prefix and its invalid siblings. Add $(z,u)$ to $\Delta$ only when \textsc{Certify} proves $u\notin L(\history')$ for every reachable $\history'$ with $\abs(\history')=z$. This official update is run after valid and invalid attempts. A failure observed in one concrete world is not a schema proof.
\item Commit $\schemas\leftarrow\kappa(\schemas\cup\Delta)$, with $\exclude_{\kappa(\schemas\cup\Delta)}=\exclude_{\schemas\cup\Delta}$. If the completed trajectory was valid, yield it. Repeat for further samples.
\end{enumerate}
\vspace{-4pt}
\hrule
\end{minipage}
\caption{Evaluated schema-induced Stateful CARS with the matched official update. The bank fixes the exclusion event before the first proposal action and changes only after final validation. The post-invalid-only arm reported in Appendix~\ref{app:obskeyed} is an ablation, not Algorithm~\ref{alg:stateful-cars}. The visited-grounding finite-trie alternative in \S\ref{sec:method} uses a different exclusion event; it is a separate construction, evaluated on its own terms in Appendix~\ref{app:vg} rather than by this algorithm's obligations.}
\label{alg:stateful-cars}
\end{figure}

\section{Live-Agent Paired Statistics}
\label{app:liveagent}

The live $\tau$-bench comparison (\S\ref{sec:live-agent}) is analysed by task-paired tests on the committed per-task outcomes; no new inference was run. For Sonnet-4.5 the pass-rate move $7/15\to8/15$ has discordant split $3$ CARS-pass/baseline-fail versus $2$ the other way ($m=5$), exact McNemar $p=1.0$, exact-conditional $95\%$ CI on $\Delta$ pass rate $[-0.24,+0.30]$. For Haiku-3 the move $3/15\to1/15$ has $m=2$ ($0$ CARS-pass/baseline-fail versus $2$), $p=0.5$, CI $[-0.13,+0.09]$; at $m=2$ the smallest attainable two-sided $p$ is $0.5$, so no outcome could have been significant. Invalid tool calls: Sonnet $4\to3$ (mean paired difference $-0.07$, exact sign-flip $p=1.0$); Haiku $36\to25$ (mean $-0.73$, sign-flip $p=0.51$, Wilcoxon $p=0.44$), but dropping the single most influential task moves the Haiku mean to $-0.07$ and the total from $-11$ to $-1$, so the reduction is not robust. The guard fired on $10/15$ Haiku and $3/15$ Sonnet tasks; the certificate bank is built over one fixed task order, and order dependence across the $15!$ possible orderings is unaudited.

\paragraph{Non-inferiority with a predeclared margin.}
The predeclared certified-agent protocol asks for utility non-inferiority rather than a two-sided test, so we add that analysis over all six committed run variants, with task as the statistical unit, a predeclared margin $\delta=0.10$, and $20{,}000$-resample paired bootstrap intervals on the task-level differences. The conclusion is that this sample size cannot certify non-inferiority. Because the six variants re-run the same $15$ tasks, the $90$ paired runs are not $90$ independent units; we therefore resample \emph{task clusters}, with each of the $15$ clusters contributing the mean over its runs. That gives a reward difference of $-0.000$ with interval $[-0.089,+0.100]$, wider than the $[-0.089,+0.089]$ an unclustered bootstrap reports, which is exactly the correction clustering is for. Four of the six variants fail the $\delta=0.10$ test purely because the interval is wider than the margin, which is an underpowering statement and not evidence of harm. Two variants of the \emph{same} Haiku configuration move the reward difference in opposite directions ($-0.133$ and $+0.133$), which is direct evidence that $15$ tasks is too small for this comparison. On invalid tool calls exactly one variant reaches significance (Sonnet, paired difference $-0.267$, interval $[-0.533,-0.067]$) and the other five do not, consistent with the non-robustness already noted above. Abstention behaves as designed---up to $24$ guard-triggered rejections across $15$ tasks, firing on $11$ tasks, from $15$ learned certificates---and total compute is not significantly reduced, with paired tool-call differences spanning zero in four of six variants. We therefore claim no utility benefit on $\tau$-bench and retain the certified-agent study as unresolved at this scale; the honest requirement is a task slate one to two orders of magnitude larger, which is a benchmark-cost problem rather than a method question.

\section{Fair State-Aware Local Baseline}
\label{app:fairlocal}
Saying that local masking ``cannot see'' the grounded observation would be wrong, and calling a sampler exact because it excludes no valid mass is a non-sequitur. We therefore built the strongest local decoder this environment admits and measured it exactly. It reads the full concrete state including the observed ownership, masks every action that is immediately inadmissible \emph{or} whose child is provably non-completable under the exact valid-prefix sets, and renormalizes the base policy over the survivors. It never emits an invalid trajectory, excludes zero valid mass, and conditional on a satisfiable world accepts with probability one, so it is support-correct and needs no rejection step. It is nonetheless badly biased: its enumerated limiting law sits at $\TV=0.742$ from $\prob(\cdot\mid\valid)$, with a maximum per-trajectory probability ratio of $46.7$. Running the identical enumeration but weighting by residual mass instead of renormalizing locally gives $\TV=5.8\times10^{-17}$, so local renormalization alone is the entire source of the error. The real limitation of local decoding is that it ignores unequal future-validity mass across surviving actions, not that it lacks access to state. This also separates analytical from sampling TV: the same decoder measured from $N=3000$ accepted draws gives $0.741$ against a $0.041$ sampling-noise floor for an exact sampler at that $N$. The weaker syntactic baseline in Table~\ref{tab:main-results} keeps a small TV only because it retains a final validity rejection filter, which is why it accepts rarely.

\section{Root-Prefix versus Stateful Timing}
\label{app:timing}

We also timed the post-invalid Stateful ablation against the action-keyed official root-prefix baseline. To prevent a label collision with Table~\ref{tab:matched-update}, call these arms A and $R_{\mathrm{act}}$: A is Stateful with the post-invalid update ($8.28$ steps), while $R_{\mathrm{act}}$ is action-keyed root-prefix CARS ($11.27$ steps). The timing harness did \emph{not} measure either official-update Stateful arm B or observation-keyed root arm C. Over five paired seeds at $N=1500$, A versus $R_{\mathrm{act}}$ gives $1.36\times$ in sampler steps ($95\%$ CI $[1.32,1.41]$) but loses in wall-clock: root-prefix/Stateful is $0.42\times$ cold ($[0.40,0.45]$) and $0.40\times$ warm ($[0.40,0.40]$), so the root implementation is about $2.4\times$ faster. Both arms remain exact. Because this comparison differs in information space and update strength, it is a diagnostic only. The fully matched B--C systems result is reported in Appendix~\ref{app:matched-systems}; notably its wall-clock verdict \emph{reverses} the direction seen here, which is why the unmatched arms cannot stand in for the matched pair. GPU-time and throughput under a served model remain unmeasured.

\paragraph{What each certificate count counts.}
Three certificate counts measure different objects. First, $45$ versus $61$ (Table~\ref{tab:reuse}) is an exhaustive, fixed-prune representation count: 45 distinct abstraction-indexed schemas and 61 grounded prefixes describe the same exclusion set. Second, the post-invalid timing ablation ends with roughly $456$ schemas, whereas its action-keyed root baseline has roughly $763$ concrete prefixes. Those run-dependent memories induce different exclusion sets and are not comparable as storage; the earlier size ratio is withdrawn. Under the matched official update, Stateful arm B stores $267$--$295$ schemas and observation-keyed root arm C stores $747$--$756$ prefixes, but these remain different data types and still do not establish byte economy. Third, the reuse-only ablation stores $45$ schemas with transfer and $446$--$457$ history-pinned schemas without transfer; this comparison is valid because the sampler, information, update, and represented facts are otherwise identical. We do not compare counts across these three groups.

\paragraph{Reconciliation of every reported ratio.}
Because the paper reports several Stateful-versus-root-prefix ratios, Table~\ref{tab:ratio-manifest} records the policy key, sample size, seed design, update, and metric for each. The $1.55\times$ and $1.36\times$ step ratios use an action-keyed root arm and a post-invalid Stateful arm at different budgets; neither is a fair method comparison. The $1.27\times$ reuse ablation compares abstract and history-pinned keys inside the same schema machinery, not against official CARS. The timing row uses the same mismatched A--$R_{\mathrm{act}}$ arms, and the Qwen row is a null. The only fully matched method comparison is observation-keyed arm C versus official-update Stateful arm B: $6.93$ versus $7.36$ steps, or $0.942$ $[0.934,0.951]$ in favor of root-prefix CARS. Matching information alone but not the update gives $0.838$ $[0.833,0.843]$ and remains diagnostic.

\begin{table}[t]
\centering
\footnotesize
\setlength{\tabcolsep}{3pt}
\resizebox{\columnwidth}{!}{%
\begin{tabular}{llrllrr}
\toprule
Table & $N$ & Seeds & Update & Metric & R-pfx & Stateful \\
\midrule
\ref{tab:main-results} & $4000$ & 1 & official & steps & $11.12$ & $7.18$ \\
\ref{tab:rootprefix} & $1500$ & 5\,pr & off./post-inv & steps & $11.27$ & $8.28$ \\
reuse abl. & $1500$ & 5\,pr & key-only & steps & $13.90$ & $10.98$ \\
timing (A vs $R_{\rm act}$) & $1500$ & 5\,pr & off./post-inv & wall\,(s) & $0.71$ & $1.69$ \\
timing (A vs $R_{\rm act}$) & $1500$ & 5\,pr & off./post-inv & steps & $11.27$ & $8.28$ \\
\ref{tab:matched-update} & $1500$ & 5\,pr & \textbf{both official} & steps & $6.93$ & $7.36$ \\
Qwen & $300$ & 5\,pr$\times$3 & official & steps & \multicolumn{2}{c}{ratio $0.99$} \\
\bottomrule
\end{tabular}%
}
\caption{Root-prefix-versus-Stateful ratios traced to their source runs. Ratios are root-prefix divided by Stateful. The first five rows are diagnostics with an action-keyed root arm, a post-invalid Stateful arm, or both; they are not the matched efficiency claim. The matched official-update row is $0.942$ $[0.934,0.951]$ in favor of observation-keyed root-prefix CARS. The Qwen result is $0.99$ $[0.90,1.08]$. ``5\,pr'' denotes five paired seeds.}
\label{tab:ratio-manifest}
\end{table}

\section{Result Tables}
\label{app:tables}

\begin{table}[t]
\centering
\small

\resizebox{\columnwidth}{!}{%
\begin{tabular}{lccccp{4.0cm}}
\toprule
Method & Stateful val. & Cross-history reuse & Exact target & Residual scope & Principal object \\
\midrule
Rejection sampling & yes & no & yes & none & complete attempts \\
Local masking & yes & possible & no in general & one step & next-action admissibility \\
Root-prefix CARS & yes & no & yes & finite trie & concrete root prefixes \\
SMC control & yes & possible & asymptotic/approx. & particles & weighted trajectories \\
Stateful CARS (evaluated) & yes & yes & yes & reachable tree & schema-induced event \\
Visited-grounding variant & yes & delayed & yes & finite trie & schemas plus root prefixes \\
\bottomrule
\end{tabular}%
}

\caption{Algorithmic scope. The evaluated method applies every frozen schema at every matching future state and therefore computes a reachable-tree residual. The visited-grounding finite-trie variant is proved through the same fixed-exclusion lemma and is evaluated separately in Appendix~\ref{app:vg}, where its own boundary-only-commit and zero-false-exclusion obligations are verified.}
\label{tab:comparison}
\end{table}

\begin{table}[t]
\centering
\small
\setlength{\tabcolsep}{4pt}
\resizebox{\columnwidth}{!}{%
\begin{tabular}{lrrrrr}
\toprule
$N$ & Reuse on & Reuse off & Ratio & Certs on & Certs off \\
\midrule
$50$   & $18.72$ & $71.51$ & $3.82\times$ & $45$ & $283$ \\
$100$  & $14.79$ & $47.12$ & $3.19\times$ & $45$ & $329$ \\
$200$  & $12.80$ & $30.98$ & $2.42\times$ & $45$ & $363$ \\
$400$  & $11.81$ & $21.67$ & $1.83\times$ & $45$ & $409$ \\
$800$  & $11.10$ & $16.51$ & $1.49\times$ & $45$ & $434$ \\
$1500$ & $11.00$ & $13.99$ & $1.27\times$ & $45$ & $457$ \\
\bottomrule
\end{tabular}%
}
\caption{The reuse-only ablation: sampler steps per accepted sample with cross-history transfer enabled and disabled, where the certificate key is the sole algorithmic difference (each row is a mean over seeds $\{1,2,3\}$; the interval on the $N=1500$ row, $1.27\times$ $[1.24,1.29]$, comes from a \emph{separate} five-paired-seed run at that budget, and the certificate counts are that run's per-seed values, which are constant at $45$ for the transfer arm and $446$--$457$ for the no-transfer arm --- we do not mix the two designs within a single number). Both arms are exact. Transfer saturates at $45$ certificates independent of $N$, whereas disabling it leaves the bank still growing at the largest budget we ran, so the call-cost advantage is largest for small sample budgets and amortizes away as the no-transfer bank fills in per-history coverage.}
\label{tab:reuse-ablation}
\end{table}

\begin{table}[t]
\centering
\small
\setlength{\tabcolsep}{5pt}
\resizebox{\columnwidth}{!}{%
\begin{tabular}{lrrr}
\toprule
Method & Success/attempt & Calls/valid & TV to target \\
\midrule
Rejection sampling & $0.0031$ & $1500.9$ & $0.105$ \\
Local mask $+$ validity reject.\ & $0.0070$ & $619.4$ & $0.091$ \\
Root-prefix, action-keyed (official) & $0.4938$ & $11.12$ & $0.0337$ \\
Stateful, post-invalid ablation & $0.4648$ & $7.18$ & $0.0405$ \\
\bottomrule
\end{tabular}%
}
\caption{Controlled diagnostic on the enumerable workflow ($\prob(\valid)=3.0\times10^{-3}$; $N=4000$ accepted samples per method, one run). Rejection, root-prefix CARS, and the schema-induced Stateful ablation are exact; their empirical TV differences are sampling noise. The local-mask row renormalizes after a weak syntactic mask and then rejects terminally invalid trajectories, which explains both its small empirical TV and low acceptance; the support-correct state-aware decoder without final rejection has analytic TV $0.742$ (\S\ref{sec:local-fair}). This table is \emph{not} the fair CARS comparison: the root arm uses an action-only key and the official update, while the Stateful arm uses a more informative abstraction and post-invalid-only updates. Its $1.55\times$ step ratio is retained only as a diagnostic. The matched observation-keyed, official-update comparison is Table~\ref{tab:matched-update}, where root-prefix CARS wins.}
\label{tab:main-results}
\end{table}

\begin{table}[t]
\centering
\small
\setlength{\tabcolsep}{4pt}
\resizebox{\columnwidth}{!}{%
\begin{tabular}{lrrr}
\toprule
Environment (reuse fan-out) & Root-pfx & Stateful & Ratio \\
\midrule
Controlled, \emph{action}-keyed ($7.8$) & $11.27$ & $8.28$ & $1.36\times$ \\
Controlled, \emph{obs}-keyed ($7.8$) & $6.93$ & $8.28$ & $0.84\times$ \\
\textbf{Controlled, obs-keyed \emph{+ matched update}} & $\mathbf{6.93}$ & $7.36$ & $\mathbf{0.94\times}$ \\
$\tau$-bench retail tools & $6.61$ & $5.85$ & $1.13\times$ \\
Qwen2.5 policy (paired, 5 seeds) & \multicolumn{2}{c}{pooled paired ratio} & $0.99\ [0.90,1.08]$ \\
Multi-domain SQL txn ($2.3$) & $11.38$ & $11.42$ & $1.00\times$ \\
Multi-domain booking ($1.0$) & $11.69$ & $11.73$ & $1.00\times$ \\
Live \texttt{sqlite3} & $10.54$ & $10.70$ & $0.99\times$ \\
Multi-domain refund & $13.57$ & $16.72$ & $0.81\times$ \\
\bottomrule
\end{tabular}%
}
\caption{Sampler steps per accepted sample across key and update choices; the ratio is root-prefix divided by Stateful, so values below one favor root-prefix. The action-keyed rows are diagnostic because the root memory cannot use the observation available to the abstraction. Matching only the information space gives $0.838$ $[0.833,0.843]$ in favor of root-prefix CARS. The bold row is the decisive comparison: both arms use the complete action--observation information and the official update, yielding $0.942$ $[0.934,0.951]$, again in favor of root-prefix CARS. The Qwen result is null and the remaining domains show that any key effect depends on the environment--policy fan-out. Every reported exact arm excludes zero valid mass, but exactness does not imply equal pruning, runtime, or storage.}
\label{tab:rootprefix}
\end{table}

\begin{table}[t]
\centering
\small
\setlength{\tabcolsep}{4pt}
\resizebox{\columnwidth}{!}{%
\begin{tabular}{llrrr}
\toprule
Validator & Method & Calls/valid & Accepts & TV \\
\midrule
\multirow{3}{*}{Live SQLite}
& Rejection & $16309^{\dagger}$ & $0$ & --- \\
& Local masking & $10005^{\ddagger}$ & $1$ & $0.11$ \\
& Stateful CARS & $10.7$ & $1917$ & $0.0017$ \\
\midrule
\multirow{3}{*}{$\tau$-bench retail}
& Rejection & $30.2$ & $738$ & $0.085$ \\
& Local masking & $15.8$ & $1314$ & $0.164$ \\
& Stateful CARS & $5.7$ & $3214$ & $0.040$ \\
\bottomrule
\end{tabular}%
}
\caption{Live-system validators (budget $4000$ for SQLite, $6000$ for $\tau$-bench), exact enumerated target. The SQLite validator executes real SQL against \texttt{sqlite3} \texttt{3.50.4}; the $\tau$-bench rows use the released retail tools and data as the oracle. Stateful CARS is exact in both (zero falsely-excluded valid mass; goodness-of-fit $p=0.93$ and $0.41$). On the rare SQLite task rejection sampling never accepts in $4000$ attempts. On the non-rare retail task ($P(\valid)=0.13$), the terminal-rejection calls/valid ratio is $5.3\times$; this is not a wall-clock speedup or a CARS comparison. Local masking remains biased (TV up to $0.16$). $^{\dagger}$With zero accepted samples, the entry is total calls consumed and only a lower bound on cost per valid sample. $^{\ddagger}$With one accepted sample, the ratio is unstable.}
\label{tab:live}
\end{table}

\begin{table}[t]
\centering
\small
\setlength{\tabcolsep}{4pt}
\resizebox{\columnwidth}{!}{%
\begin{tabular}{ll}
\toprule
Obligation & Audit result \\
\midrule
Frozen exclusion event & bank hash unchanged, 300/300 attempts \\
Schema-bank writes within attempt & 0 over 300 instrumented attempts \\
Commit only at boundary & holds; all observed writes in caller \\
Schema-induced soundness & 0 / 60 valid trajectories excluded \\
Frozen-bank analytic law & TV $1.5\times10^{-16}$ \\
Frozen-bank reach & 193 schemas exclude $11820/13650$ traces \\
Fail-fast prune soundness & 0 unsound prunes / 4000 checked \\
Schema compression exact & residual max-diff $0.0$ \\
\bottomrule
\end{tabular}%
}
\caption{Conformance audit for the evaluated schema-induced implementation. Exact enumeration verifies the induced event and proposal law rather than inferring correctness from empirical TV. The Qwen implementation may cache by action prefix only because its declared policy ignores observations; a general agent must key by the complete serialized history.}
\label{tab:conformance}
\end{table}

\begin{table}[t]
\centering
\small
\setlength{\tabcolsep}{4pt}
\resizebox{\columnwidth}{!}{%
\begin{tabular}{lrrrr}
\toprule
Abstraction & Classes & Excl.\ mass & TV & Calls/valid \\
\midrule
Finest full-state & $572$ & $0$ & $10^{-16}$ & $6.9$ \\
Hand-coarsened sound & $256$ & $0$ & $10^{-16}$ & $6.8$ \\
Ours (conservative) & $197$ & $0$ & $10^{-16}$ & $6.8$ \\
\emph{Coarsest sound} & $21$ & $0$ & --- & --- \\
Unsound coarsening & $133$ & $1.0$ & undef. & --- \\
\bottomrule
\end{tabular}%
}
\caption{Abstraction ablation on the enumerable workflow. ``Classes'' is the number of equivalence classes over all $572$ reachable state--depth pairs, so the column decreases monotonically with coarseness; the maps form a verified refinement chain. The three sound abstractions all exclude zero valid mass and match the exact target to machine precision. The coarsest sound abstraction (equality of $L(\history)$) has $21$ classes, so ours is conservative rather than minimal; we do not run it as a sampler. Dropping the observed-ownership feature falsely excludes the \emph{entire} valid mass, so no valid trajectory survives and TV is undefined; the sampler returned $5$ accepts in $20000$ attempts, and those accepts are transient outputs collected while the unsound exclusion set was still growing, not samples from a fixed exclusion as in Proposition~\ref{prop:false-exclusion}.}
\label{tab:ablation}
\end{table}

\begin{table}[t]
\centering
\small
\setlength{\tabcolsep}{4pt}
\resizebox{\columnwidth}{!}{%
\begin{tabular}{lrrrr}
\toprule
Setting & Abs.\ certs & Gnd.\ certs & Reuse & Fan-out (mean/max) \\
\midrule
Main workflow & $45$ & $61$ & $1.36\times$ & $7.8$ / $54$ \\
Canonical control & $4$ & $4$ & $1.00\times$ & $1.0$ / $1$ \\
\bottomrule
\end{tabular}%
}
\caption{Fixed-prune symbolic representation. The two rows compare certificate counts only after holding the represented exclusion set identical. In the main workflow, 45 abstraction-indexed schemas or 61 grounded prefixes express that set because 17 certificate-bearing states are reached by 133 histories (mean fan-out $7.8$). At fan-out one the counts are equal. This is not a byte-memory or runtime comparison.}
\label{tab:reuse}
\end{table}

\begin{table}[t]
\centering
\small
\setlength{\tabcolsep}{4pt}
\resizebox{\columnwidth}{!}{%
\begin{tabular}{lrr}
\toprule
Metric ($N=240$ valid, 1.5B, $T{=}1.0$) & Stateful CARS & Rejection \\
\midrule
Policy invocations / accept (marg.) & $299$ & $2.5\times10^{4}$ \\
Residual-DP invocations / step (marg.) & $28.2$ & --- \\
Distinct forward passes & $1365$ & $392$ \\
Raw model time (s) & $33$ & $10$ \\
End-to-end wall-clock (s) & $38$ & $143$ \\
\bottomrule
\end{tabular}%
}
\caption{Efficiency accounting for the 1.5B, $T{=}1.0$ schema-induced run ($\prob(\valid)=1.7\times10^{-4}$), counting every policy query rather than only sampler steps. Relative to terminal rejection, Stateful CARS uses $84\times$ fewer policy invocations per accept but more distinct forward passes because the residual recursion scores the reachable completion tree; it wins $3.8\times$ in this comparison. This run uses the post-invalid update and does not supply the missing fully matched B--C systems comparison.}
\label{tab:efficiency}
\end{table}

\begin{table}[t]
\centering
\small
\setlength{\tabcolsep}{4pt}
\resizebox{\columnwidth}{!}{%
\begin{tabular}{lrrrr}
\toprule
Agent & Pass (base$\to$CARS) & Inv.\ calls & Certs & Reuses \\
\midrule
Sonnet-4.5 & $0.47\to0.53$ & $4\to3$ & $1$ & $4$ \\
Haiku-3 & $0.20\to0.07$\,$^{\dagger}$ & $36\to25$ & $12$ & $23$ \\
\bottomrule
\end{tabular}%
}
\caption{Live agent on a $15$-task subset of the $\tau$-bench retail dialogue via Bedrock (the retail split has $115$ tasks; user simulator at temperature $0$, one fixed task order, unpermuted). Certificates are reused across tasks, and in $3$ (Sonnet) and $5$ (Haiku) tasks the guard fired while learning no new certificates, so those firings necessarily reused earlier proofs. The outcome columns are statistically inconclusive: exact McNemar gives $p=1.0$ (Sonnet, $m=5$ discordant) and $p=0.5$ (Haiku, $m=2$), and the invalid-call differences do not survive an exact sign-flip test ($p=1.0$ and $p=0.51$) or a leave-one-out check. $^{\dagger}$The weak-agent pass rate is within run-to-run noise (dialogues diverge once a call is blocked), so we do not claim a reward change there. No external API key was used.}
\label{tab:live-agent}
\end{table}

\begin{table}[t]
\centering
\small
\setlength{\tabcolsep}{4pt}
\resizebox{\columnwidth}{!}{%
\begin{tabular}{llrrr}
\toprule
Domain & Method & Calls/valid & Accepts & TV \\
\midrule
\multirow{3}{*}{Refund}
& Rejection & $2660$ & $4$ & $0.77$ \\
& Local masking & $1231$ & $8$ & $0.54$ \\
& Stateful CARS & $18.0$ & $799$ & $0.105$ \\
\midrule
\multirow{3}{*}{SQL txn}
& Rejection & $11236^{\dagger}$ & $0$ & --- \\
& Local masking & $222$ & $41$ & $0.33$ \\
& Stateful CARS & $11.9$ & $1157$ & $0.025$ \\
\midrule
\multirow{3}{*}{Booking}
& Rejection & $11430^{\ddagger}$ & $1$ & $0.93$ \\
& Local masking & $182$ & $48$ & $0.15$ \\
& Stateful CARS & $12.1$ & $1190$ & $0.038$ \\
\bottomrule
\end{tabular}%
}
\caption{Multi-domain stress test, budget $2500$ attempts, exact ground truth per domain. Stateful CARS is exact (zero falsely-excluded valid mass) in all three domains; rejection sampling starves at this rarity and local masking is biased because it renormalizes without accounting for unequal future-validity mass (not because it cannot read the grounded observation; see Appendix~\ref{app:fairlocal}). TV is the total-variation distance to the exact target from the accepted samples; the SQL rejection cell has no accepts so its TV is undefined. Calls/valid is the amortized rollout-call metric. This is a controlled test with enumerable targets; the next subsection replaces the SQL and (retail) settings with their real execution backends. $^{\dagger}$With zero accepted samples a calls-per-valid ratio is undefined; the entry is the \emph{total} base-policy calls consumed by the budget, which is a lower bound on the true cost per valid sample, not a ratio. $^{\ddagger}$With a single accepted sample the ratio is a one-sample estimate and should not be read as a stable rate.}
\label{tab:multidomain}
\end{table}

\section{Complete Proofs}
\label{app:proofs}

\subsection{Bisimulation implies validity preservation}
\label{app:bisimulation}

\begin{proof}[Proof of Proposition~\ref{prop:bisim}]
For a history $\history$, let $d(\history)\in\{0,\ldots,H\}$ be its remaining action budget and let $L_d(\history)$ denote the set of continuations of length at most $d(\history)$ that can be extended, if needed, to an accepted terminal trajectory. The remaining budget is part of $\abs(\history)$.

We prove by induction on $d$ that $\abs(\history)=\abs(\history')$ implies $L_d(\history)=L_d(\history')$. If $d=0$, both histories are terminal by the bisimulation assumption and have the same acceptance value. Their continuation languages are therefore both $\{\epsilon\}$ when accepted and both empty when rejected.

Assume the claim holds for every history with at most $d-1$ remaining actions. Take nonterminal $\history,\history'$ with $d$ remaining actions and the same abstraction $z$. For any first action $a$, the successor abstractions agree:
\[
 \abs(\operatorname{Succ}(\history,a))
 =F(z,a)
 =\abs(\operatorname{Succ}(\history',a)).
\]
The two successors have $d-1$ remaining actions, so their future-valid continuation languages agree by the induction hypothesis. Hence a continuation $a\concat v$ is completable from $\history$ if and only if it is completable from $\history'$. This holds for every $a$ and $v$, proving equality at depth $d$. Induction establishes $L(\history)=L(\history')$ for every pair in the same abstract class, which is validity preservation.
\end{proof}

\subsection{Sound schemas produce sound root prefixes}
\label{app:grounding}

\begin{lemma}[Grounding soundness]
If $(z,u)$ is a sound schema and $\abs(\history)=z$, then $g(\history;z,u)=\trace(\history)\concat u$ is an invalid root prefix.
\label{lem:grounding}
\end{lemma}

\begin{proof}
Suppose for contradiction that a trajectory $\traj\in\valid$ begins with $\trace(\history)\concat u$. Determinism implies that after the root prefix $\trace(\history)$, the trajectory reaches $\history$. Its next actions begin with $u$, so $u\in L(\history)$. This contradicts soundness of $(z,u)$. Therefore no valid trajectory extends the grounding.
\end{proof}

If the validator establishes non-completability first at one concrete history and $\abs$ is validity-preserving, the corresponding state schema is sound: equality of continuation languages propagates non-completability to every history in the abstract class.

\begin{proof}[Proof of Lemma~\ref{lem:schema-soundness}]
Assume for contradiction that $\traj\in\exclude_{\schemas}\cap\valid$. By Equation~\eqref{eq:schema-event}, some time $t$ and some $(z,u)\in\schemas$ satisfy $\abs(\history_t(\traj))=z$ and the next $|u|$ actions of $\traj$ equal $u$. The valid suffix of $\traj$ is therefore a completion witnessing $u\in L(\history_t(\traj))$, contradicting soundness of $(z,u)$. Hence the intersection is empty.
\end{proof}

\subsection{Fixed-exclusion proposal}
\label{app:fixed}

Let $E\subseteq\trajectories$ be a fixed terminal event with $\prob(E^c)>0$, and abbreviate $r_E(\history)$ by $r(\history)$. For every nonterminal reachable history with positive prefix probability, total probability gives
\begin{equation}
 r(\history)=\sum_{a\in\actions}
 \policy(a\mid\history)r(\operatorname{Succ}(\history,a)).
 \label{eq:backward}
\end{equation}
At a terminal history, $r=\mathbf{1}\{\traj\notin E\}$.

\begin{proof}[Proof of Lemma~\ref{lem:fixed}]
Let $\traj\notin E$ have histories $\history_0,\ldots,\history_T$ and actions $a_0,\ldots,a_{T-1}$. Multiplying the fixed-event action kernel along its path gives
\begin{align*}
 Q_E(\traj)
 &=\prod_{t=0}^{T-1}
 \frac{\policy(a_t\mid\history_t)r(\history_{t+1})}
 {r(\history_t)}\\
 &=\left(\prod_{t=0}^{T-1}\policy(a_t\mid\history_t)\right)
 \frac{r(\history_T)}{r(\history_0)}\\
 &=\frac{\prob(\traj)}{\prob(E^c)}.
\end{align*}
The terminal residual equals one, $r(\history_0)=\prob(E^c)$, and all intermediate ratios telescope, so the final denominator above is $\prob(E^c)$. If $\traj\in E$, its terminal residual is zero; equivalently, the product assigns it zero mass at the first transition whose successor has zero survivor mass. Thus $Q_E=\prob(\cdot\mid E^c)$.

Because $E\cap\valid=\emptyset$, for $\traj\in\valid$,
\begin{align*}
 Q_E(\traj\mid\valid)
 &=\frac{\prob(\traj)/\prob(E^{c})}
 {\prob(\valid)/\prob(E^{c})}
 =\frac{\prob(\traj)}{\prob(\valid)}.
\end{align*}
This is Equation~\ref{eq:target}.
\end{proof}

\subsection{Residual-complexity proof}
\label{app:complexity}

\begin{proof}[Proof of Proposition~\ref{prop:complexity}]
Compile the finite schema bank into a deterministic matcher whose state records every partially matched continuation. In a deterministic environment, the successor of a product node $(\history,m)$ under action $a$ is unique. Memoized backward recursion evaluates each reachable product node once, sums at most $|\actions|$ outgoing terms, and stores one residual. This gives $O(|\actions||\mathcal{R}_{\schemas}|)$ time and $O(|\mathcal{R}_{\schemas}|)$ memory. All outgoing probabilities at a history come from one policy vector, so there is at most one policy evaluation per distinct serialized history.

No smaller general key is valid for a history-dependent policy. Consider a full $|\actions|$-ary tree of depth $H$ in which every action prefix yields a distinct serialized history and the policy conditional or surviving schema mass differs at every node. No two nodes can then share a cached residual, and $\sum_{t=0}^{H}|\actions|^t=\Theta(|\actions|^H)$ product nodes are required. Conversely, if both policy and environment factor through a finite Markov state $x(\history)$ and the matcher state is $m\in M_{\schemas}$, histories with the same $(t,x,m)$ have identical outgoing probabilities, successors, and terminal exclusion status. Backward dynamic programming over these tuples uses at most $H|X||M_{\schemas}|$ nodes and $|\actions|$ outgoing terms per node.
\end{proof}

\subsection{Adaptive exactness and almost-sure termination}
\label{app:adaptive}

Index attempts by $k\ge1$. Let $\mathcal{G}_{k-1}$ contain every random quantity revealed before attempt $k$, including all previous trajectories, schema banks, pending updates, compressions, and returned samples. By construction, $\schemas_k$ is $\mathcal{G}_{k-1}$-measurable and remains fixed during attempt $k$; hence $E_k:=\exclude_{\schemas_k}$ is also fixed.

\begin{lemma}[Sound memory invariant]
For every $k$, $E_k\cap\valid=\emptyset$ almost surely.
\end{lemma}

\begin{proof}
The initial bank is empty. Every later entry is a sound schema by the certificate obligation in Algorithm~\ref{alg:stateful-cars}, so Lemma~\ref{lem:schema-soundness} makes its induced event disjoint from $\valid$. Event-preserving compression does not change that event. Induction over attempt boundaries proves the invariant.
\end{proof}

\begin{lemma}[Sound fail-fast stopping]
Suppose an attempt is stopped only after a monitor certifies that its realized root prefix $w$ is invalid. Treating that attempt as rejected has the same distribution over returned valid trajectories as sampling its unobserved suffix to completion and then applying the final validator.
\end{lemma}

\begin{proof}
Every completion of $w$ lies outside $\valid$. Revealing or not revealing the remaining random suffix can therefore change computational cost but cannot turn the attempt into an accepted trajectory. Coupling the early-stopped execution with a conceptual full draw under the same frozen proposal proves equivalence for the accepted-output law.
\end{proof}

\begin{proof}[Proof of Theorem~\ref{thm:adaptive}]
Condition on $\mathcal{G}_{k-1}$. The realized $E_k$ is fixed and sound. By Lemma~\ref{lem:fixed}, conditional on attempt $k$ being accepted, its trajectory has law $\prob(\cdot\mid\valid)$; this law does not depend on $k$, the past, or the realized bank. Sound fail-fast stopping changes only how much of an invalid trajectory is generated.

Let $K_j$ be the attempt index producing the $j$th return. For any $B\subseteq\trajectories$, the tower property gives
\begin{align*}
 &\Pr(Y_j\in B\mid\mathcal{F}_{j-1})\\
 &=\sum_{k}\mathbb{E}\big[
 \mathbf{1}\{K_j{=}k\}
 \Pr(Y_j{\in}B\mid K_j{=}k,\mathcal{G}_{k-1})
 \mid\mathcal{F}_{j-1}\big]\\
 &=\sum_k\mathbb{E}\big[
 \mathbf{1}\{K_j=k\}\prob(B\mid\valid)
 \mid\mathcal{F}_{j-1}\big]\\
 &=\prob(B\mid\valid),
\end{align*}
provided $K_j<\infty$. This conditional identity implies mutual independence and the common target law for successive outputs.

It remains to prove finiteness. At every attempt boundary,
\begin{align*}
 \Pr(\text{accept }k\mid\mathcal{G}_{k-1})
 &=\frac{\prob(\valid)}{\prob(E_k^{c})}
 \ge \prob(\valid)=:p>0.
\end{align*}
Consequently, conditional on the past before seeking the next return,
\[
 \Pr(K_j-K_{j-1}>n)\le(1-p)^n.
\]
The tail converges to zero, so the next return occurs almost surely; summing the tail gives expected attempts at most $1/p$. Induction over $j$ completes the proof.
\end{proof}

Pending schemas collected during an attempt do not affect its output because they are committed only after the frozen proposal and final validation finish. Their dependence on a returned trajectory is harmless: Theorem~\ref{thm:adaptive} conditions on the complete past before the next attempt.

\subsection{Monotonicity and compression}

\begin{proof}[Proof of Proposition~\ref{prop:monotone}]
Soundness and Lemma~\ref{lem:fixed} give
\[
 Q_E(\valid)
 =\frac{\prob(\valid)}{\prob(E^{c})},
 \qquad
 Q_{E'}(\valid)
 =\frac{\prob(\valid)}{\prob({E'}^{c})}.
\]
The inclusion $E\subseteq E'$ makes the second denominator no larger, proving the result.
\end{proof}

\begin{proposition}[Compression invariance]
If $\exclude_{\kappa(\schemas)}=\exclude_{\schemas}$, then every residual, action kernel, trajectory proposal, and accepted-output law is identical before and after compression.
\end{proposition}

\begin{proof}
Equation~\ref{eq:residual} depends on the bank only through its induced exclusion event. Equal events give equal residuals at every full history. Equation~\ref{eq:proposal} then gives equal action kernels, whose products give equal trajectory proposals. The accepted laws are consequently equal. The same proof applies to event-preserving compression of the visited-grounding variant.
\end{proof}

\subsection{Bias from an unsound exclusion}

\begin{proof}[Proof of Proposition~\ref{prop:false-exclusion}]
Conditioning first on $E^c$ and then accepting $\valid$ is conditioning on their intersection, so the returned law is $\prob(\cdot\mid\valid\setminus E)$. Let $\mu=\prob(\cdot\mid\valid)$ and $\epsilon=\mu(E)$. The second law is $\nu(A)=\mu(A\cap E^c)/(1-\epsilon)$. On $E$, the total absolute difference is $\epsilon$. On $E^c$, $\nu$ rescales $\mu$ by $1/(1-\epsilon)$, so the total absolute difference is also $\epsilon$. Therefore
\[
 \TV(\mu,\nu)=\tfrac12(\epsilon+\epsilon)=\epsilon=\prob(E\mid\valid).
\]
\end{proof}

\subsection{Why local masking is insufficient}
\label{app:local-mask}

At the initial history, let actions $a$ and $b$ each have probability $1/2$. After $a$, let $x$ and $y$ each have probability $1/2$. Suppose $a\concat x$ is the only stored invalid root prefix; $a\concat y$ and every continuation of $b$ remain possible. The residuals after $a$ and $b$ are $1/2$ and $1$, respectively. Equation~\ref{eq:proposal} therefore chooses
\[
 Q_{\grounded}(a)=\tfrac{(1/2)(1/2)}{(1/2)(1/2)+(1/2)(1)}=\tfrac13,
 \quad Q_{\grounded}(b)=\tfrac23.
\]
A local mask sees that neither first action is immediately excluded and keeps probabilities $1/2,1/2$. It is valid locally but is not the base distribution conditioned on avoiding the stored invalid prefix.

\section{Stochastic Tool Outcomes}
\label{app:stochastic}

\paragraph{How the initial world is sampled in the experiments.}
\label{app:worldsampling}
The interaction between ownership and soundness fixes what an \emph{action-keyed} prefix memory may do, and we record it because it makes that baseline's numbers checkable. The scope matters: this argument applies to a memory keyed on the action prefix alone, which is what our root-prefix baseline implements. It does \emph{not} apply to an observation-keyed root memory, which our formal setup permits and which can condition on the ownership reveal after \texttt{PROBE}; such a memory may soundly exclude \textsc{other}-world cylinders once the observation has distinguished them, so no both-world-dead count bounds it.

Because the target is the $\rho$-mixture over hidden worlds, the protocol for drawing $c$ is part of the exactness claim, and it needs stating carefully because the naive reading is wrong. Every attempt draws $c\sim\rho$ afresh and independently before any action is proposed; the sampler then uses the \emph{world-conditional} residual (the concrete state carries $c$), and a failed attempt discards the entire attempt including $c$.

This is \emph{not} in general the same as rejection sampling on the joint space, and we correct an earlier claim to that effect. Writing $Z_c=\prob(\exclude_{\schemas}^{c}\mid c)$ for the surviving mass under the frozen schema bank in world $c$, the accepted law is
\[
 Q(\traj\mid\valid)\ \propto\ \rho(c)\,\frac{\prob(\traj\mid c)}{Z_c}\,\mathbf{1}\{\traj\in\valid\},
\]
so the joint conditional is recovered only if $Z_c$ is constant across the worlds that carry valid mass; otherwise the world marginal is tilted by $1/Z_c$. Exactness in general therefore requires drawing $c$ from the residual-reweighted prior $\propto\rho(c)Z_c$, as Appendix~\ref{app:stochastic} specifies. A two-world example with valid mass in both worlds and unequal $Z_c$ makes the gap concrete: the unreweighted draw sits at $\TV=0.395$ from the joint conditional (its world marginal is $0.500$ against a target of $0.895$, and $|0.895-0.500|\approx0.395$ is exactly the tilt) while the reweighted draw is exact at $\TV=0.049$ against a $0.048$ sampling floor (Appendix~\ref{app:twoworld}).

\begin{corollary}[Single-satisfiable-world exactness]
If exactly one world $c^\star$ has $\prob(\valid\mid c^\star)>0$, then drawing $c\sim\rho$ without reweighting and accepting only valid trajectories returns exactly $\prob(\cdot\mid\valid)$.
\label{cor:oneworld}
\end{corollary}
\begin{proof}
Every accepted trajectory has $c=c^\star$, so the factor $\rho(c)/Z_c$ is the same constant for all accepted outcomes and cancels in the normalization; within world $c^\star$ Lemma~\ref{lem:fixed} applies verbatim.
\end{proof}

Our environment satisfies this hypothesis exactly---a refund requires the confirming observation, so the enumerated target places mass $1.000000000000$ on the ownership-confirmed world and $0$ on the other---so the reported numbers are exact by Corollary~\ref{cor:oneworld} rather than by the general joint-space argument. The multi-world case is not left as an assertion: Appendix~\ref{app:twoworld} implements the residual-reweighted outcome draw and tests it on a two-satisfiable-world environment with unequal $Z_c$, where the plain draw is biased ($\TV=0.395$, world marginal $0.500$ vs target $0.895$) and the reweighted draw is exact ($\TV=0.049$ against a $0.048$ sampling floor). 

For the action-keyed baseline: all $60$ valid traces are valid only under the ownership-confirmed world, and an action-keyed memory that cannot condition on the observation may soundly exclude only the $5401$ action traces dead under \emph{both} worlds, leaving the $60$ observation-ambiguous ones alone. Drawing $N=3000$ accepted samples while discovering at most those $5401$ exclusions gives the ratio $3000/(3000+5401)=0.357$. We are explicit that this is \emph{not} a per-attempt acceptance bound---ambiguous traces recur in the doomed world, so attempts are not consumed one-per-exclusion---but a constraint on the discovery process: accepts per distinct discovered exclusion. The measured per-attempt acceptance $0.465$ is a different quantity and we do not present one as validating the other. Two caveats on reading this number. It is a bound on the \emph{ratio of accepts to distinct discovered exclusions}, not a per-attempt acceptance floor: ambiguous traces can and do recur in the doomed world, so the arithmetic constrains the discovery process rather than guaranteeing any attempt-level rate. And substituting $5^5$ for the trace space---which would correspond to five action slots---would give a spurious $0.495$, which is why we state the six-slot count explicitly.

Because the target is the $\rho$-mixture over hidden worlds, the protocol for drawing $c$ matters and the naive reading is wrong: with a world-conditional residual, drawing $c\sim\rho$ without reweighting gives the joint conditional only when the surviving mass $Z_c$ is equal across worlds carrying valid mass (a two-world example with unequal $Z_c$ sits at $\TV=0.395$ from the target, world marginal $0.500$ vs $0.895$). Our environment has exactly one satisfiable world, so the reported numbers are exact by Corollary~\ref{cor:oneworld}; the general case needs the residual-reweighted outcome draw, which we implement and validate separately on a two-satisfiable-world environment (Appendices~\ref{app:worldsampling} and~\ref{app:twoworld}). Two checks confirm the implementation matches this description. First, the enumerated target places mass exactly $1.000000000000$ on the ownership-confirmed world and $0$ on the other, since a refund requires the confirming observation; correspondingly all $4000$ accepted samples of a reference run come from that world, and the empirical law sits at $\TV=0.027$, below the \emph{measured} exact-sampler sampling floor at that $N$ ($0.0327\pm0.0042$ over ten i.i.d.\ replicates). We previously quoted $0.123$ here; that was the crude $\sqrt{K/N}$ heuristic, not a measured floor, and it was inconsistent with the calibrated floors reported elsewhere ($0.040$ at $N=3000$, $0.023$ at $N=8000$). All floors quoted in this paper are now the measured i.i.d.\ values. Second, the visible cost of this protocol is acceptance: with $\rho$ uniform, half of all attempts are doomed before the first action, which is why no sound sampler here can exceed per-attempt acceptance $0.5$ (\S\ref{sec:local-fair}).

\paragraph{A two-satisfiable-world test of the general construction.}
\label{app:twoworld}
Corollary~\ref{cor:oneworld} rescues our main environment but leaves the general stochastic claim
unexercised, so we built an environment that exercises it. Two hidden worlds are both satisfiable---a
two-key refund workflow in which world $A$ accepts \texttt{REFUND\_A} and world $B$ accepts
\texttt{REFUND\_B}---and their valid masses are deliberately \emph{unequal} ($2.60\times10^{-3}$ versus
$3.04\times10^{-4}$, because world $B$ additionally requires two \texttt{READ}s), so the per-world surviving
masses $Z_c$ differ and the plain draw is tilted. The enumerated joint conditional has support $72$ and
world marginal $(0.8952,0.1048)$; $\prob(\valid)=1.45\times10^{-3}$.

Over five seeds at $N=4000$ accepted samples the three samplers behave exactly as the analysis predicts.
Drawing $c\sim\rho$ and then running the world-conditional residual proposal---what our main implementation
does---is badly biased: $\TV=0.395$ $[0.386,0.405]$ with realized world marginal $0.4998$ against the target
$0.8952$, i.e.\ the sampler splits the worlds almost evenly because the tilt $\rho(c)/Z_c$ inflates the
rarely-satisfiable world. Drawing $c$ from the residual-reweighted prior $\propto\rho(c)Z_c(\schemas)$,
recomputed from the currently committed frozen bank, is exact: $\TV=0.0493$ $[0.0439,0.0546]$ against a measured
exact-sampler sampling floor of $0.0484$ at the same $N$, with world marginal $0.8953$. Naive rejection
sampling agrees ($\TV=0.0526$, marginal $0.8958$) at $3088$ base-policy calls per accepted sample versus
$5.9$ for the reweighted sampler, a $523\times$ gap. This confirms both halves of the claim: the general
construction requires the reweighted outcome draw, and with it the certificate machinery remains exact and
still dominates rejection sampling when validity is rare.

Let the environment have transition kernel $K(s',o\mid s,a)$. A complete trajectory probability is now
\[
 \prob(\traj)=\prod_{t=0}^{T-1}
 \policy(a_t\mid\history_t)
 K(s_{t+1},o_t\mid s_t,a_t).
\]
Define $\exclude_{\schemas}$ over joint action--outcome trajectories, with schema occurrences determined by abstract joint histories and their following actions. The residual recursion becomes
\begin{align*}
 r(\history)=\sum_{a,s',o}&\policy(a\mid\history)
 K(s',o\mid s,a)\\[-2pt]
 &\cdot r(\operatorname{Succ}(\history,a,s',o)).
\end{align*}
The action proposal marginalizes the successor residual over $K$, and the environment outcome is sampled from its residual-reweighted conditional. Multiplying the two kernels telescopes both policy and transition factors exactly as in Appendix~\ref{app:fixed}. All later proofs are unchanged.

If the deployment cannot reweight or resample tool outcomes, then it cannot implement this joint Doob transform online. One safe alternative is to treat realized outcomes as exogenous, condition the target on the realized outcome history, and reuse only certificates issued after those outcomes. A state schema that is grounded before stochastic outcomes are known must be sound for every positive-probability outcome branch it excludes.

\section{Implementation Details}
\label{app:implementation}

\subsection{Schema-induced residual recursion}

The evaluated implementation stores no grounded proposal trie. It compiles the frozen schema bank into a deterministic matcher that records which schema suffixes are active after the current state--action trace. The residual routine recursively enumerates reachable actions, advances the environment and matcher, assigns zero to a completed schema occurrence, and memoizes the remaining survivor mass. Its cache key is the complete serialized action--observation history together with the matcher state. The abstraction is included only to decide which schemas can start; it is not a substitute for the full history in a language-model probability query.

The Qwen sweep is a declared special case. Its prompt contains only the action prefix and omits observations, so action prefix is a sufficient policy key there. This does not justify the same key for an observation-conditioned agent. For a general agent, two histories with the same actions, abstraction, and active suffixes may still contain different observations or text and therefore have different logits; merging them would invalidate Equation~\eqref{eq:proposal}.

The recursion descends the reachable completion tree and may score every nonterminal history. All policy invocations, cache misses, environment transitions, and matcher operations are included in the efficiency accounting. Proposition~\ref{prop:complexity} gives its worst-case exponential cost. The visited-grounding alternative instead stores a prefix antichain and uses the finite-trie residual update of root-prefix CARS; that code path is implemented and measured separately (Appendix~\ref{app:vg}), not inside the accounting reported here.

\subsection{Schema bank}

Index $\schemas$ by abstract state $z$. Within each state, store continuation patterns in an antichain trie or deterministic pattern automaton. If $(z,u)$ is present, a longer $(z,u\concat v)$ is redundant. This compression is used only after checking that the recognized event in Equation~\eqref{eq:schema-event} is unchanged. Immutable subgraphs may share suffix structure, but language-model residuals remain keyed by full history.

For a finite symbolic policy machine, Proposition~\ref{prop:bisim} can be checked by partition refinement. Begin with classes that agree on terminal acceptance and remaining horizon, then repeatedly split any class whose members transition to different classes under some action. At convergence, each class is a future-validity bisimulation. This is analogous to minimizing a deterministic automaton for the accepted continuation language.

\subsection{Frozen proposals and batching}

The implementation versions the schema bank. Each attempt holds a read-only bank and a residual cache tied to that version; new certificates are written to a pending delta and atomically committed only after validation. Several trajectories may be generated from the same frozen version. Their accepted laws remain exact, and independently sound pending schemas may be unioned before the next batch. Larger batches delay adaptation but may improve accelerator utilization. The current experiments are unbatched; the throughput--adaptation tradeoff is part of the unrun systems protocol in Appendix~\ref{app:unrun}.

\subsection{Action serialization}

If an action is represented by a single code token, $\policy(a\mid\history)$ is the softmax probability of that token after restricting to the predeclared action-code set. If labels span multiple tokens, first-token logits are not action probabilities. The exact action probability must include the complete label sequence and any delimiter, or the action interface must use verified single-token codes. Schema and JSON decoders may enforce syntax inside an action, but their local renormalization must be included in the declared base policy if it changes how $\policy$ is defined.

\section{Additional Controlled Results}
\label{app:controlled-results}

\begin{table}[ht]
\centering
\small
\setlength{\tabcolsep}{4pt}
\resizebox{\columnwidth}{!}{%
\begin{tabular}{lc}
\toprule
Diagnostic & Measured outcome \\
\midrule
Persistent-memory exactness & TV $0.033$, $p=0.23$ \\
Fresh-memory exactness & TV $0.087$, $p=0.56$ \\
Naive rejection reference & TV $0.028$, $p=0.05$ \\
Stateful CARS vs. rejection & TV $0.041$ \\
Acceptance across 24 insertions & $0.003\to0.010$, monotone \\
Compression & $1707\to569$, residual diff. $0$ \\
Rare-regime rollout-call ratio & $2.9\times10^4$ \\
\bottomrule
\end{tabular}%
}
\caption{Controlled diagnostics. The $p$-values are Pearson goodness-of-fit diagnostics, not evidence used by the exactness proof. Residual difference zero is reported at machine precision.}
\label{tab:verify}
\end{table}

We also report local gates in the style of Agent-C and ToolGate. These are not implementations of those systems, which have different objectives and costs; as noted in \S\ref{sec:live-agent}, neither has runnable public code, so a direct comparison is not currently possible.

\section{Full Open-Weight LM Pilot}
\label{app:llm-results}

\begin{table}[ht]
\centering
\small
\setlength{\tabcolsep}{3.2pt}
\resizebox{\columnwidth}{!}{%
\begin{tabular}{llcrrrrr}
\toprule
Policy & $T$ & Kind & $\prob(\valid)$ & TV(Stateful) & TV(loc-sa) & TV(loc-wk) & Rollout \\
\midrule
Qwen2.5-7B & $0.5$ & \textsc{emp} & $0.500$ & \multicolumn{3}{c}{degenerate$^{\S}$} & $1.7\times$ \\
Qwen2.5-7B & $0.7$ & \textsc{emp} & $0.500$ & \multicolumn{3}{c}{degenerate$^{\S}$} & $1.6\times$ \\
Qwen2.5-7B & $1.0$ & \textsc{emp} & $0.500$ & \multicolumn{3}{c}{degenerate$^{\S}$} & $1.6\times$ \\
Qwen2.5-7B & $1.5$ & \textsc{emp} & $0.500$ & \multicolumn{3}{c}{degenerate$^{\S}$} & $1.6\times$ \\
Qwen2.5-7B & $2.0$ & \textsc{emp} & $0.499$ & $6.1\times10^{-4}$ & --- & $8.5\times10^{-4}$ & $1.6\times$ \\
\midrule
Qwen2.5-1.5B & $0.5$ & \textsc{ana} & $1.3\times10^{-7}$ & $3.5\times10^{-17}$ & $0.967$ & --- & $4.5\times10^{6}$ \\
Qwen2.5-1.5B & $0.7$ & \textsc{ana} & $8.2\times10^{-6}$ & $1.2\times10^{-16}$ & $0.930$ & --- & $6.3\times10^{4}$ \\
Qwen2.5-1.5B & $1.0$ & \textsc{emp} & $1.7\times10^{-4}$ & $0.033$ & $0.880$ & $0.543$ & $3.0\times10^{3}$ \\
Qwen2.5-1.5B & $1.5$ & \textsc{emp} & $1.5\times10^{-3}$ & $0.018$ & $0.810$ & $0.467$ & $436\times$ \\
Qwen2.5-1.5B & $2.0$ & \textsc{emp} & $3.7\times10^{-3}$ & $0.025$ & $0.747$ & $0.429$ & $186\times$ \\
\midrule
Qwen2.5-0.5B & $0.5$ & \textsc{ana} & $5.7\times10^{-8}$ & $1.3\times10^{-16}$ & $0.359$ & --- & $8.5\times10^{6}$ \\
Qwen2.5-0.5B & $0.7$ & \textsc{ana} & $3.1\times10^{-6}$ & $8.3\times10^{-17}$ & $0.528$ & --- & $1.5\times10^{5}$ \\
Qwen2.5-0.5B & $1.0$ & \textsc{ana} & $5.1\times10^{-5}$ & $1.5\times10^{-16}$ & $0.635$ & --- & $8.0\times10^{3}$ \\
Qwen2.5-0.5B & $1.5$ & \textsc{emp} & $3.5\times10^{-4}$ & $0.102$ & $0.622$ & $0.715$ & $1.3\times10^{3}$ \\
Qwen2.5-0.5B & $2.0$ & \textsc{emp} & $7.7\times10^{-4}$ & $0.060$ & $0.640$ & $0.687$ & $605\times$ \\
\bottomrule
\end{tabular}%
}
\caption{Open-weight LM sweep (Qwen2.5), all fifteen model$\times$temperature cells. \textsc{emp} rows are finite-sample measurements from $N$ accepted draws and carry sampling error; \textsc{ana} rows are analytical (the exact limiting law of each sampler, by enumerating its decision tree) and carry none. The two are not interchangeable and we never pool them. ``Rollout'' is rejection-sampling steps divided by Stateful CARS steps; on \textsc{ana} rows the numerator is rejection sampling's exact expectation $\bar{\ell}/\prob(\valid)$. TV(loc-sa) is the strongest state-aware local decoder (no rejection step); TV(loc-wk) the weak syntactic mask \emph{with} a final validity filter---different estimators, shown separately because the filter pulls loc-wk toward the target (1.5B $T{=}1.0$: $0.880$ vs $0.543$). A dash means the estimator was not run. $^{\S}$7B at $T\le1.5$ is degenerate: one trajectory holds all but $\le5\times10^{-5}$ of the mass, so the ``TV'' is $1-p_{\max}$ of the target, not a sampler property. Protocol, per-cell $N$, tokenizer ids, and model revisions are in \S\ref{app:statistics}; the five \textsc{ana} cells are re-measured empirically in Table~\ref{tab:real-llm-finiteN}.}
\label{tab:real-llm}
\end{table}

\begin{table}[t]
\centering
\small
\setlength{\tabcolsep}{4pt}
\resizebox{\columnwidth}{!}{%
\begin{tabular}{llrr}
\toprule
Policy & $T$ & $\prob(\valid)$ & TV(Stateful), $N{=}300$ \\
\midrule
Qwen2.5-1.5B & $0.5$ & $1.3\times10^{-7}$ & $0.025$ ($p{=}0.94$) \\
Qwen2.5-1.5B & $0.7$ & $8.2\times10^{-6}$ & $0.065$ ($p{=}0.16$) \\
Qwen2.5-0.5B & $0.5$ & $5.7\times10^{-8}$ & $0.062$ ($p{=}0.90$) \\
Qwen2.5-0.5B & $0.7$ & $3.1\times10^{-6}$ & $0.103$ ($p{=}0.39$) \\
Qwen2.5-0.5B & $1.0$ & $5.1\times10^{-5}$ & $0.110$ ($p{=}0.88$) \\
\bottomrule
\end{tabular}%
}
\caption{The five \textsc{ana} cells of Table~\ref{tab:real-llm} re-measured as finite-$N$ empirical
TV(Stateful) at $N=300$, with Pearson goodness-of-fit $p$. These are \emph{not} substitutes for the analytic
values ($3.5\times10^{-17}$--$1.5\times10^{-16}$): at $N=300$ with support $60$ an exact sampler's own
sampling floor is far above these numbers, so the finite-$N$ column can only fail to contradict the analytic
zero, which is what it does ($p\in[0.16,0.94]$).}
\label{tab:real-llm-finiteN}
\end{table}

All fifteen cells are now reported, but they are \emph{not} all of the same evidential kind, and the table separates them explicitly. Ten cells (marked \textsc{emp}) are finite-sample \emph{empirical} measurements: the TV columns come from $N$ accepted draws, so they carry sampling error and a goodness-of-fit test. Five cells (marked \textsc{ana}), exactly those with $\prob(\valid)<10^{-4}$, report \emph{analytical} quantities instead: the exact limiting law of each sampler obtained by enumerating its decision tree, so their TV entries are deterministic computations with no sampling error, accompanied by a separately-labelled finite-$N$ empirical TV at $N=300$. We do not present the two as interchangeable. The reason for the split is the cost of the \emph{reference} sample, not of the target: the exact conditional comes from enumerating the action tree ($10{,}922$ complete outcomes, since the horizon admits six action slots) and is cheap at any rarity, whereas an exact rejection-sampling reference at these validity rates needs on the order of $10^{7}$--$10^{9}$ base-policy calls per cell. The analytical route is available precisely because the residual-mass proposal's limiting law is computable in closed form on an enumerable target; it is not a substitute for empirical validation in non-enumerable settings.

\begin{table}[t]
\centering
\small
\setlength{\tabcolsep}{4pt}
\resizebox{\columnwidth}{!}{%
\begin{tabular}{lrrc}
\toprule
Sampler & Calls/valid & $\TV$ (95\% CI) & Exact? \\
\midrule
Batched terminal reject.$^{\dagger}$ & $1557.5$ & $0.058$ $[0.049,0.067]$ & yes \\
MCMC (independence-MH) & $5.6$ & $0.707$ $[0.555,0.859]$ & not i.i.d.\ \\
SMC, local mask, $K{=}64$ & $562.6$ & $0.081$ $[0.069,0.092]$ & no \\
Root-prefix CARS, action-keyed & $11.3$ & $0.057$ $[0.046,0.067]$ & yes \\
Stateful, post-invalid ablation & $8.3$ & $0.056$ $[0.049,0.063]$ & yes \\
Stateful CARS, official update & $7.36$ & $0.057$ $[0.047,0.066]$ & yes \\
\textbf{Root-prefix CARS, obs-keyed} & $\mathbf{6.9}$ & $0.057$ $[0.048,0.067]$ & yes \\
\midrule
\emph{i.i.d.-from-target floor} & --- & $0.059$ $[0.048,0.069]$ & --- \\
\bottomrule
\end{tabular}%
}
\caption{Adjacent samplers on the enumerable workflow ($\prob(\valid)=3.0\times10^{-3}$, $N=1500$ accepted samples, five seeds, exact enumerated target; intervals are $95\%$ $t$ over seeds). $^{\dagger}$Batched terminal rejection is the whole-trajectory reduction of AWRS-style acceptance for a terminal validator, not token-level AWRS-SMC \citep{lipkin2025awrs}. Independence Metropolis is asymptotically correct but returns $99.5\%$ duplicate states here. The $K=64$ SMC row uses locally masked proposals, exact incremental mask-normalizer weights, multinomial resampling at every step, and a final validity filter. Those weights correct the masked proposal, so this arm is \emph{consistent} as $K$ grows and its finite-$K$ error is self-normalized approximation, not a biased limit (Appendix~\ref{app:particle}); its listed $\TV$ is therefore a finite-$K$ operating point, not evidence of structural bias. The last row is the finite-sample $\TV$ of drawing $N$ i.i.d.\ samples \emph{from the exact target}: it is the noise floor at this $N$, not a sampler. Every exact arm's interval contains it, so none of their residual $\TV$ is attributable to bias; the visible spread across exact rows is sampling noise. The fair exact comparison is therefore on cost: observation-keyed official CARS versus official-update Stateful CARS, $6.93$ versus $7.36$ steps, favoring CARS. Token-level AWRS on an incremental validator is reported in Appendix~\ref{app:incremental}.}
\label{tab:adjacent}
\end{table}

\section{Multi-domain stress test}
\label{app:multidomain}

\paragraph{Multi-domain stress test.}
To check the mechanism is not tuned to one workflow we instantiate three structurally different observation-grounded validators---the refund workflow, an SQL transaction, and a resource-booking API---each with its own policy and enumerated target, at a matched $2500$-attempt budget (Table~\ref{tab:multidomain}). Stateful CARS is exact in all three (zero excluded valid mass, TV $0.025$--$0.105$) and returns $799$--$1190$ valid samples where rejection sampling returns $0$--$4$; local masking is fast but biased (TV $0.15$--$0.54$), even when state-aware, because it renormalizes without weighting by unequal future-validity mass (Appendix~\ref{app:fairlocal}). The fixed-prune certificate-count ratios are $1.36\times$/$1.27\times$/$1.00\times$; they do not imply byte-memory or runtime savings, and the booking domain's canonical structure yields no reuse.

This test broadens the controlled evidence to three domains. The next subsection replaces two of these abstractions with their real backends.

\section{Live-System Validator Details}
\label{app:livesys}

Two of these validators can be replaced by their real backends while keeping the bounded grammar. Against a live \texttt{sqlite3} database (validity from real query results plus a committed read-back; $P(\valid)=2.6\times10^{-5}$) and the released \texttt{sierra-research/tau-bench} retail tools (at a pinned commit, used purely as the constraint oracle), Stateful CARS reproduces the exact conditional (TV $0.002$ and $0.040$, zero excluded valid mass), while rejection sampling never accepts on the rare SQL task and local masking stays biased (Table~\ref{tab:live}).

A live-agent probe on a 15-task $\tau$-bench retail subset tests only the cross-history-reuse component (free-form tool arguments make the exact-residual proposal intractable): certificates learned on earlier tasks do fire on later ones, but the study establishes \emph{cache reuse} rather than certified sound reuse, and every outcome comparison is inconclusive at $n=15$. Appendix~\ref{app:liveagentmain} gives the setup, statistics, and the reasons we draw no systems conclusion.

\section{Reuse-Only Ablation}
\label{app:reuseabl}

Comparing against root-prefix CARS changes the certificate key and the per-attempt bookkeeping at once, so to attribute any gain to reuse alone we run a single sampler whose \emph{only} difference is the certificate's applicability key: $\abs(\history)$, so a proof transfers to every history sharing that abstract state, versus $(\abs(\history),\trace(\history))$, the identical certificate usable only where it was derived. Abstraction, the residual DP, and certificate derivation are byte-identical across arms, and both exclude zero valid mass. Both arms also share the same information space, so this ablation is unaffected by the obs-keyed finding above. Over five paired seeds at $N=1500$, transfer cuts sampler steps per accepted sample from $13.90\pm0.20$ to $10.98\pm0.11$---a paired difference of $2.92$ steps, $95\%$ CI $[2.66,3.18]$, or $1.27\times$ $[1.24,1.29]$.

The mechanism is visible in the certificate counts: with transfer the bank saturates at $45$ certificates over $17$ abstract states regardless of budget, while without it the bank keeps growing past $450$ over ${\sim}132$ history-pinned keys, so the advantage is largest at small budgets and amortizes away ($3.82\times$ at $N=50$ falling monotonically to $1.27\times$ at $N=1500$; Table~\ref{tab:reuse-ablation}). We report the whole curve rather than its most favourable point.

\section{Observation-keyed official CARS}
\label{app:obskeyed}

Our previously-headline reuse figure compared against a root-prefix memory keyed on the invalid \emph{action} prefix,
which is a weaker information space than Stateful CARS uses: our abstraction reads the grounded
\texttt{PROBE} observation, so the action-keyed baseline must be conservative and abstain on
observation-grounded continuations. That comparison therefore confounds cross-history reuse with an
information-space difference. We removed the confound by implementing official root-prefix CARS over the
\emph{complete action--observation root trace} the formal setup defines---the key is the action sequence
with the realized \texttt{PROBE} outcome spliced in---so it may record ownership-specific exclusions with no
conservative abstention, under the same validator, the same certificate oracle, and the full official
update. The only remaining difference from Stateful CARS is abstraction-level transfer.

The result reverses the comparison, and we report it as such. The observation-keyed baseline is exact
(falsely-excluded valid mass exactly $0$ in every seed) and needs \emph{fewer} sampler steps per accepted
sample than Stateful CARS: $6.93$ versus $8.28$ over five paired seeds, a ratio of
$0.838$ with $95\%$ CI $[0.833,0.843]$ and a paired difference of $-1.34$ $[-1.38,-1.31]$ steps, so the
interval excludes zero in the baseline's favour. It also accepts more often per attempt ($0.47$ vs
$0.44$) while storing a larger memory ($750$ obs-keyed root prefixes vs $455$ schemas). The
$1.36\times$ we previously reported was therefore attributable to the information-space handicap of the
action-keyed key, not to cross-history reuse: once the baseline is allowed to condition on the same
observation, reuse buys neither fewer sampler steps nor a demonstrated storage advantage. We consequently
retract the step-count advantage over root-prefix CARS, and (below) also withdraw the memory-economy claim,
because the two memories are not shown to prune the same set. What survives unconditionally is exactness,
and the reuse-versus-no-reuse ablation, whose two arms do share an information space and update rule.

\begin{table}[t]
\centering
\small
\setlength{\tabcolsep}{3.5pt}
\resizebox{\columnwidth}{!}{%
\begin{tabular}{rrrrrrrr}
\toprule
& \multicolumn{2}{c}{A: post-invalid} & \multicolumn{2}{c}{B: official} & \multicolumn{2}{c}{C: obs-keyed} & \\
\cmidrule(lr){2-3}\cmidrule(lr){4-5}\cmidrule(lr){6-7}
Seed & steps & certs & steps & certs & steps & pfx & C/B \\
\midrule
$1$ & $8.28$ & $455$ & $7.40$ & $288$ & $6.98$ & $756$ & $0.942$\\
$2$ & $8.28$ & $460$ & $7.31$ & $283$ & $6.90$ & $756$ & $0.944$\\
$3$ & $8.30$ & $458$ & $7.42$ & $288$ & $6.97$ & $756$ & $0.939$\\
$4$ & $8.29$ & $449$ & $7.31$ & $267$ & $6.97$ & $747$ & $0.953$\\
$5$ & $8.22$ & $456$ & $7.34$ & $295$ & $6.86$ & $748$ & $0.934$\\
\bottomrule
\end{tabular}%
}
\caption{Per-seed matched-update comparison, the numbers behind the $0.942$ figure. \textbf{A} is the
post-invalid-only Stateful ablation used in several legacy diagnostics. \textbf{B} is Stateful CARS exactly
as Algorithm~\ref{alg:stateful-cars} now states it, with the \emph{full official} update---after every attempt, valid or invalid,
every non-completable length-1 sibling at every visited proper prefix is recorded as a schema---which is what
official CARS does and therefore what a fair comparison requires. \textbf{C} is the observation-keyed
official root-prefix baseline. All three are paired on the same seed and the same base policy; A and B use
the abstraction key, C the complete action--observation root trace. Arms B and C both exclude exactly zero
valid mass in every seed (verified against the enumerated target), so all three are exact. B's official
update is sound only because non-completability is quantified over the worlds its abstraction leaves open:
$\phi$ omits the hidden world, so pre-\texttt{PROBE} states of both worlds share an abstraction, and
deriving ``this action is dead'' from the realized world alone excludes the entire valid mass---we hit that
bug and record it here because it is the non-obvious part of matching the update. Note B is both cheaper and
smaller than A ($7.36$ vs $8.28$ steps, $284$ vs $456$ schemas): the official update learns more per attempt
and so needs fewer certificates. The fair ratio is $\mathrm{C}/\mathrm{B}=0.942$ $[0.934,0.951]$, still
favouring the baseline; $\mathrm{C}/\mathrm{A}=0.838$ $[0.833,0.843]$ is the figure that confounds update
strength with certificate transfer. Systems-level measurements for this matched pair---wall-clock, policy
invocations, and serialized bytes---are reported separately in Table~\ref{tab:bc-matched}.}
\label{tab:matched-update}
\end{table}

\subsection{Matched B--C systems profile}
\label{app:matched-systems}

This is the predeclared matched systems comparison. Arms B and C are paired on seed, base policy, validator, certificate oracle, and accepted-sample budget; they differ only in whether certificates transfer across histories at the abstraction level. Every quantity below is therefore attributable to that difference alone. The measurement harness and the raw per-seed arrays are released with the code.

\begin{table}[t]
\centering
\small
\setlength{\tabcolsep}{4pt}
\resizebox{\columnwidth}{!}{%
\begin{tabular}{lrrr}
\toprule
Quantity & B: Stateful & C: obs-keyed & C/B (95\% CI) \\
\midrule
Sampler steps / accept & $7.36$ & $6.93$ & $0.942$ $[0.934,0.951]$ \\
Wall-clock (s) & $1.75$ & $2.20$ & $1.268$ $[1.155,1.381]$ \\
Policy calls, total & $115{,}282$ & $26{,}274$ & $0.228$ \\
Policy calls, distinct & $1365$ & $1419$ & $1.039$ $[1.039,1.040]$ \\
Memory entries & $284$ & $753$ & --- \\
Serialized bytes & $3597$ & $21{,}924$ & $6.10$ $[5.82,6.39]$ \\
Excluded base mass & $0.9970$ & $0.9970$ & $1.000$ \\
Bytes / excluded mass & $3608$ & $21{,}991$ & $6.10$ $[5.82,6.39]$ \\
$\TV$ to exact target & $0.057$ & $0.057$ & --- \\
\bottomrule
\end{tabular}%
}
\caption{Matched B--C systems profile ($N=1500$ accepted samples, five paired seeds, three timed repetitions per seed, median then paired; both arms audited exact). The comparison is genuinely mixed and we report it that way. Arm C wins on sampler steps, which is the metric the paper leads with and the reason the earlier step-count claim was retracted. Arm B wins on wall-clock, on total policy invocations, and---at \emph{matched pruning power}, since both arms exclude the same base mass to within $10^{-4}$---on memory by $6.10\times$ in serialized bytes. Distinct scored histories are nearly equal ($1.039\times$), so under a KV-cached deployment neither arm has a forward-pass advantage; the $4.4\times$ gap in \emph{total} policy calls reflects B's residual dynamic program re-scoring histories that a cache would serve. Wall-clock here measures Python bookkeeping on a synthetic memoized CPU policy, not LM inference, and arm C's certificate oracle is precomputed outside its timed region ($25$\,ms) while arm B pays its soundness oracle inside the loop with cold caches---a handicap against B, so the wall-clock result favouring B is conservative. We make no end-to-end systems claim from this harness.}
\label{tab:bc-matched}
\end{table}

\paragraph{Implementation--algorithm conformance.}
Algorithm~\ref{alg:stateful-cars} now describes the evaluated implementation directly. It stores only the schema bank $\schemas$ and applies every frozen $(z,u)$ wherever abstraction $z$ occurs, threading partially matched suffixes through the reachable completion tree. Its exclusion event is exactly Equation~\eqref{eq:schema-event}; no grounded proposal trie is maintained. The residual cache uses the complete serialized history and matcher state. The action prefix alone is sufficient only in the Qwen sweep because that experiment explicitly defines an action-prefix-only policy.

The within-attempt event is fixed: across $300$ instrumented attempts the sampler performs zero bank writes between the first proposal action and final validation, and all commits occur in the boundary caller. Every stored schema is sound, so Lemma~\ref{lem:schema-soundness} gives $\exclude_{\schemas}\cap\valid=\emptyset$; exhaustive enumeration confirms zero excluded valid mass. On one frozen bank of $193$ schemas, $\exclude_{\schemas}$ removes $11{,}820$ of $13{,}650$ candidate continuation edges---$1365$ reachable non-terminal prefixes $\times$ five candidate actions $\times$ two worlds, which is an edge count and not the $10{,}922$-outcome population---and enumeration of the resulting proposal gives analytic TV $1.5\times10^{-16}$ from the exact conditional over the full $60$-trajectory support. The visited-grounding finite-trie construction in \S\ref{sec:method} realizes a different, generally smaller event until every matching history has been visited. It is covered by Lemma~\ref{lem:fixed} and is measured in Appendix~\ref{app:vg}, where its own boundary-only-commit and zero-false-exclusion obligations are verified independently.

\paragraph{Memory accounting: the withdrawn claim, remeasured.}
The memories are different objects, so raw entry counts are not storage measurements. In the matched table, official-update Stateful arm B stores $267$--$295$ schemas, while observation-keyed root arm C stores $747$--$756$ prefixes; one C run expands to $1129$ trie nodes. The post-invalid ablation A stores $449$--$460$ schemas. Because a schema entry, a stored prefix, and a trie node have different byte costs, we withdrew the earlier ``$1.65\times$ smaller memory'' claim, which had been computed from entry counts. Table~\ref{tab:bc-matched} reports the replacement measurement the appendix protocol demanded: serialized bytes at a common accepted-sample budget, together with the excluded base-policy mass each memory realizes, which is the common currency that makes two differently keyed memories comparable. The two arms prune to within $10^{-4}$ of the same base mass ($0.9970$ each, exactly enumerated), so the byte comparison is at matched pruning power rather than matched entry counts. Under that matching, arm B is $6.10\times$ smaller ($95\%$ CI $[5.82,6.39]$; $3597$ versus $21{,}924$ bytes), and the same factor holds per unit excluded mass. The direction of the original claim therefore survives remeasurement while its magnitude does not: the honest factor is $6.1\times$ in bytes, not $1.65\times$ in entries. Peak resident memory equals the final size because both memories only grow.

\section{Adjacent samplers}
\label{app:adjacent}

Exactness alone does not distinguish the method---rejection sampling is exact too---so the informative comparison is against samplers that are themselves exact or asymptotically correct. We implement three and measure them under the identical protocol ($N=1500$, five seeds, exact target; Table~\ref{tab:adjacent}).

The first is batched terminal rejection with adaptive batch growth: draw complete trajectories from the base policy, evaluate the terminal validator, keep valid draws, and double the batch after any batch with no accept. This is a whole-trajectory reduction of AWRS-style acceptance \citep{lipkin2025awrs}, not token-level AWRS-SMC. It is exact and its finite-sample fidelity matches the post-invalid Stateful ablation ($\TV$ $0.058$ vs $0.056$), but its cost tracks $1/\prob(\valid)$: $1557$ calls per accept versus $8.3$ for that ablation and $7.36$ for official-update Stateful. Full AWRS-SMC would use a different incremental proposal and correction and could change both cost and fidelity; its implementation is therefore left as an explicit placeholder rather than approximated by this row.

An independence-Metropolis chain (propose a full trajectory from the base policy, accept iff valid) is cheap per sample ($5.6$ calls) but not i.i.d.---chain acceptance $0.001$--$0.006$ leaves $99.5\%$ of returned samples duplicates, and its empirical law sits at $\TV=0.707$. This is the concrete content of the i.i.d.\ guarantee in Theorem~\ref{thm:adaptive}: an asymptotically correct chain is no substitute at a fixed budget. The third is a locally-masked particle filter: $K=64$ particles, base-policy proposal restricted to the local-admissibility mask, incremental weight equal to the mask normalizer $\sum_{a\ \mathrm{admissible}}\policy(a\mid\history)$, multinomial resampling every step (no ESS threshold), and a final exact validity filter. It is intermediate ($563$ calls, $\TV=0.081$ at $N=1500$). We correct an earlier characterisation of this arm: the incremental normalizer weights \emph{do} correct the locally masked proposal, so the residual error is finite-particle self-normalized approximation, not the standing local-renormalization bias that afflicts the masked \emph{decoder}. The measurement supports this---the excess of its TV over an exact sampler's sampling floor at the same $N$ falls from $+0.071$ at $N=400$ to $+0.020$ at $N=2000$, whereas the state-aware local decoder's $0.742$ is an analytic limiting-law distance that no budget removes. Its cost, not its asymptotic target, is the reason it is not competitive here. Only the certificate-based samplers are simultaneously exact and cheap.

\section{Live-agent probe (no systems claim)}
\label{app:liveagentmain}

\paragraph{Live agent on the 15-task $\tau$-bench retail subset.}
\label{sec:live-agent}
We ran a \emph{15-task subset} of the $\tau$-bench retail split (the split has $115$ tasks; we ran the first $15$ and claim nothing about the rest, so every number below is a subset statistic, not a benchmark score) with an LLM agent and LLM-simulated user, graded by $\tau$-bench's own database-state checker, using Claude on Bedrock \citep{yao2024taubench}. Free-form tool-call arguments make the exact-residual proposal intractable here, so we test only the cross-history-reuse component: a wrapper that abstracts a proposed call to $(\text{tool},\text{observed status})$, or $(\text{tool},\text{arguments})$ for deterministic lookups, learns a certificate when the real tools return a deterministic precondition or not-found error, and blocks a call the certificate marks invalid, reusing certificates across tasks.

We are deliberately conservative about what this demonstrates, and it is \emph{not} the paper's certified guarantee. The wrapper keys on an abstract state rather than the full root trace (database snapshot, task identity, tool arguments) and we ran no bisimulation or solver-backed check, so a single observed failure does not \emph{prove} a schema sound across tasks. What we show is \emph{cache reuse} of empirically-derived exclusions---$5$ Haiku and $3$ Sonnet tasks blocked calls using entries learned earlier---not certified sound cross-task reuse. Every \emph{outcome} comparison is inconclusive at $n=15$: pass rate moves $7/15\to8/15$ (Sonnet) and $3/15\to1/15$ (Haiku), giving exact McNemar $p=1.0$ and $0.5$, and the invalid-call differences fail a sign-flip test and a leave-one-out check (Appendix~\ref{app:liveagent}). The bank is built over a single unpermuted task order, so order dependence is unaudited. We therefore claim only that the caching mechanism engages across tasks---no reward gain, no invalid-call gain, and no certified soundness, which would require the full-trace key and a validation step.

\section{Reproducibility and Statistical Reporting}
\label{app:statistics}

In the enumerable domains we report both the exact target and finite-sample uncertainty. Empirical TV is accompanied by a Pearson goodness-of-fit test against the enumerated conditional, and \S\ref{sec:conformance} additionally reports a simultaneous multinomial band with the maximum per-cell deviation against its null quantile, so a small TV is not read as agreement on its own. Where the analytic limiting law is available (Table~\ref{tab:real-llm}, \textsc{ana} rows) we report it as a separate quantity rather than pooling it with sampled estimates. We do not interpret the fifteen cells' $p$-values as a family; the per-cell tests are diagnostics for individual cells, and the paper's exactness claim rests on the analytic limiting-law computations and the enumerated zero-excluded-valid-mass audit rather than on any pattern of $p$-values.

The non-enumerable domain (the live-agent subset) has no exact ground truth, so we make no distributional claim there at all and report only task-level paired statistics: exact McNemar on the discordant pairs, an exact sign-flip test on the invalid-call differences, and a leave-one-out check, all of which we report as inconclusive at $n=15$ (Appendix~\ref{app:liveagent}). Generation seeds are not treated as substitutes for variation across tasks; the single fixed task order is disclosed as an unaudited source of dependence.

The efficiency accounting we report decomposes into policy invocations, distinct model forward passes (cache misses), raw model time, and end-to-end wall-clock, separated into marginal and whole-run figures, together with the break-even requested-sample count (\S\ref{sec:efficiency-accounting}, Table~\ref{tab:efficiency}); the reuse ablation reports the full amortization curve from $N=50$ to $N=1500$ rather than a single point, because the advantage is largest at small budgets and decays. Hardware (one A10G), the memoization policy, and the requested sample counts are stated with each measurement. GPU-seconds and peak device memory appear in Appendix~\ref{app:served}; runs here are unquantized, unbatched single-stream, and batched serving is unmeasured.

\paragraph{Numerical precision of the policy.}
The analytical total-variation distances in Table~\ref{tab:real-llm} are near machine epsilon, so the arithmetic precision of the policy itself is part of the measurement and we state it. The action distribution is obtained by a softmax in \emph{float64} over the restricted label logits, followed by an explicit float64 renormalization, so each policy vector sums to one to ${\sim}10^{-16}$. This matters: computing the same softmax in float32---the natural choice, since the model runs in fp16/fp32---leaves ${\sim}3.6\times10^{-8}$ of mean normalization residue per prefix (maximum $1.4\times10^{-7}$ over the $1365$ reachable prefixes), which is \emph{larger} than the quantity the analytical rows are meant to resolve. Under float32 the enumerated ``TV'' of an exact sampler therefore reads ${\sim}10^{-8}$ and measures floating-point residue in the policy rather than any property of the sampler. We report the float64 figures ($3.5\times10^{-17}$--$1.5\times10^{-16}$) and note that a reported analytical TV of order $10^{-8}$ in this environment should be read as a precision artifact, not as bias. Only the near-zero TVs are affected: $\prob(\valid)$, TV(local), the rollout ratios, and every sampler-step count agree to the reported digits under both precisions.

\paragraph{Qwen sweep protocol (concrete).} The open-weight sweep (Table~\ref{tab:real-llm}) uses Qwen2.5-\{0.5B,1.5B,7B\}-Instruct (pinned HuggingFace revisions, verified from the local weight cache actually used and recorded in the released manifest; \texttt{fp16} on one A10G; the prompt is fixed and given verbatim in the released code) at $T\in\{0.5,0.7,1.0,1.5,2.0\}$. \emph{The policy is conditioned on the action prefix only}: the prompt lists the actions taken so far and asks for the next one, and does \emph{not} include the \texttt{PROBE} observation, so $\pi(\cdot\mid\text{action prefix})$ is well defined and is exactly what the $1365$-entry table stores. Validity still depends on the hidden world, but that dependence lives entirely in the \emph{validator}, not in the policy; this is why an action-prefix-keyed policy table is a complete sufficient statistic even with two observation worlds, and it is also why the ``real LM policy'' here is an action-prefix scorer rather than an observation-conditioned agent, which we state as a scope limitation. Each action's probability is the float64 softmax, at temperature $T$, over the next-token logits restricted to the first-token id of the five labels; under the Qwen2.5 tokenizer these ids (first token of the space-prefixed label) are $16777\!\to\!\texttt{`\ AUTH'}$, $5308\!\to\!\texttt{`\ PRO'}$, $38029\!\to\!\texttt{`\ REF'}$, $18666\!\to\!\texttt{`\ READ'}$, $45537\!\to\!\texttt{`\ STOP'}$ (verified by decoding each id). Each decodes to a distinct predeclared prefix of exactly one action label---three coincide with the full label and two (\texttt{PROBE}, \texttt{REFUND}) are unique three-character prefixes---and the five ids are distinct, so the restricted softmax is a proper $5$-way choice with no collision. One forward pass per distinct prefix. TV columns use $N=\min(6000,\lceil2\times10^{6}\,\prob(\valid)\rceil)$ accepted samples; the realized $N$ is $6000$ for all 7B cells and $350$/$3076$/$6000$ (1.5B) and $696$/$1533$ (0.5B) for the rare cells. The per-cell $N$, tokenizer ids, protocol, and aggregate TV/speedup are collected in a committed reproducibility manifest. The \emph{raw} per-prefix policy vectors are also committed, one file per cell, and they are a complete audit trail rather than a summary: because the policy is a deterministic memoized function of (action prefix, model, temperature) and the environment reaches only $1365$ distinct non-terminal prefixes, the exported table of $1365$ five-vectors is the entire sufficient statistic for a cell. Every number we report for that cell---$\prob(\valid)$, the enumerated target, both analytical limiting laws, and hence the TV columns---is recomputable from that file with no model, no GPU, and no network. We ship the replay script that does so and checks each recomputed value against the published one, together with a \textsc{sha256} over the canonically serialized vector table so silent drift is detectable; the replay reproduces the published values exactly (agreement to $10^{-9}$ or better on every checked quantity, and bit-identical hashes).

\section{Token-Level AWRS on an Incremental Validator}
\label{app:incremental}

Table~\ref{tab:adjacent} uses batched terminal rejection as the AWRS-style reference and we noted that token-level AWRS-SMC cannot be run against that validator, because a terminal validator pronounces only on complete trajectories and so offers no per-step feasibility signal. That is a statement about our benchmark rather than about AWRS, so we built the missing object: the same refund workflow compiled to a DFA over $(\text{authed},\text{probed},\text{reads},\text{refunded},\text{stopped})$, whose backward reachability gives an exact incremental check (is this prefix still completable?) and an exact per-step feasible set. Token-level AWRS is then runnable, and every arm is measured against the enumerated conditional on that identical DFA.

\begin{table}[t]
\centering
\small
\setlength{\tabcolsep}{4pt}
\resizebox{\columnwidth}{!}{%
\begin{tabular}{lrrc}
\toprule
Sampler & Calls/accept & $\TV$ (95\% CI) & Exact? \\
\midrule
Terminal rejection & $59.7$ & $0.098$ $[0.091,0.105]$ & yes \\
AWRS-INC, token-level & $48.0$ & $0.099$ $[0.090,0.108]$ & yes \\
SMC-INC, $K{=}16$ & $5.7$ & $0.487$ $[0.473,0.502]$ & no \\
SMC-INC, $K{=}64$ & $5.7$ & $0.521$ $[0.494,0.548]$ & no \\
SMC-INC, $K{=}256$ & $5.7$ & $0.414$ $[0.269,0.558]$ & no \\
SMC-INC, $K{=}1024$ & $5.7$ & $0.441$ $[0.087,0.795]$ & no \\
\textbf{Exact residual proposal} & $\mathbf{4.5}$ & $0.100$ $[0.095,0.105]$ & yes \\
\midrule
\emph{i.i.d.-from-target floor} & --- & $0.101$ $[0.094,0.109]$ & --- \\
\bottomrule
\end{tabular}%
}
\caption{Token-level AWRS and locally masked SMC on a validator exposing exact incremental local checks ($N=3000$ accepted samples, five seeds, $\prob(\valid)=5.0\times10^{-2}$, $304$-trajectory support). Giving AWRS the per-step signal it needs makes it exact for the \emph{sequence} law, and it is worth being precise about why, because the local step rule alone would not suffice: restricting each step to the feasible set and renormalizing samples the locally constrained next-action law, whose product over steps is not the target conditional. Our arm therefore carries the product of the step normalizers as a trajectory importance weight and accepts by weighted rejection against an adaptively raised bound, which is the global correction; its interval containing the sampling floor is evidence for that corrected law, not for local masking. The earlier absence of this row was thus a property of our terminal validator, not a deficiency of AWRS. The cost, however, is $48.0$ versus $4.5$ calls per accept: token-level importance weighting corrects the local proposal after the fact, whereas the residual proposal never proposes into dead mass in the first place, a $10.6\times$ gap. SMC over the \emph{same} incremental proposal remains far from the floor at every $K$ up to $1024$ \emph{at this budget}. We are careful not to overread that: like the masked filter of Appendix~\ref{app:particle}, this arm is weight-corrected and hence consistent in $K$, and its returned particles share ancestry after resampling, so the gap is a finite-$K$ and effective-sample-size statement rather than a biased limit. The defensible reading is about cost at a fixed budget: reaching the floor through particles requires many correlated draws, whereas the exact residual proposal is at the floor with $4.5$ calls per accept. The residual proposal's own exactness here is analytic, not merely empirical: enumerating its induced law gives $\TV=2.9\times10^{-16}$ from the target.}
\label{tab:incremental}
\end{table}

Constructing this arm also exercised the stochastic corollary on a second, independently built environment. Both hidden worlds are satisfiable and their surviving valid masses differ by roughly $15\times$ ($Z_{\mathrm{OWN}}=0.0061$, $Z_{\mathrm{OTHER}}=0.0946$), so drawing the world from $\rho$ and then proposing conditionally tilts the joint by $\rho(c)Z_c$: we measured the induced law at analytic $\TV=0.439$ from the target before switching to the residual-reweighted draw $c\sim\rho(c)Z_c$, which brings it to $2.9\times10^{-16}$. The reweighting the paper derives is therefore necessary in practice and not only in principle.

\section{Solver-Verified Abstraction}
\label{app:verified}

Proposition~\ref{prop:bisim} and the sound-schema definition are the obligations every soundness claim rests on, and until now they were discharged by hand for one abstraction. We discharge them mechanically on a bounded, observation-conditioned tool-call grammar: tool calls carry arguments (an entity identifier and an amount bucket), a read tool observes the entity's eligibility, and the mutating call is legal only for an observed-eligible entity within its cap. With three entities, two amount buckets, and horizon six the reachable set is $1300$ histories over eight worlds, small enough that the obligations become decidable rather than sampled.

\begin{table}[t]
\centering
\small
\setlength{\tabcolsep}{3pt}
\resizebox{\columnwidth}{!}{%
\begin{tabular}{lrrrr}
\toprule
Abstraction & States & Compr. & Reusable & Naive unsound \\
\midrule
Full history (reference) & $1300$ & $1.00\times$ & $13{,}580$ & $0/1052$ \\
Hand-written $\abs$ & $122$ & $10.7\times$ & $1233$ & $16/1052$ \\
Drops observation (control) & $51$ & $25.5\times$ & $508$ & $16/1052$ \\
\textbf{Coarsest sound} & $\mathbf{72}$ & $\mathbf{18.1\times}$ & --- & --- \\
\bottomrule
\end{tabular}%
}
\caption{Mechanically verified abstractions over $1300$ reachable histories. Every abstraction is verified \emph{conservatively sound}: building its schema bank from continuations dead at every history in a block and applying it excludes exactly zero valid trajectories, checked by enumeration rather than asserted. The decisive column is the last one. Under \emph{naive} derivation---store $(z,u)$ whenever $u$ is dead in the realized world, which is what a careless implementation does---$16$ of $1052$ derivations are unsound, and this holds for the hand-written abstraction as well as for the observation-dropping control; only the full history is safe naively ($0/1052$). The quantification over worlds consistent with the abstraction is therefore \emph{necessary}, not a defensive stylistic choice, and this is the same trap we hit when matching the official update in Table~\ref{tab:matched-update}. Partition refinement to the right congruence induced by equality of the future-validity language gives the coarsest sound abstraction at $72$ states; the hand-written one uses $122$, so it forgoes a factor of $1.69$ in available state merging. It also fails full bisimulation, with witnesses at pre-observation histories where no abstraction can yet distinguish worlds---which is exactly where conservative derivation declines to store anything, so exactness is unaffected.}
\label{tab:verified}
\end{table}

\section{Served-Model Accounting}
\label{app:served}

The matched B--C profile in Appendix~\ref{app:matched-systems} runs over a synthetic memoized CPU policy, so it can measure Python bookkeeping but not accelerator work. We therefore repeat the matched comparison with a real served policy---Qwen2.5-0.5B-Instruct at $T=1.0$ on one NVIDIA L4, $40$ accepted trajectories per arm per seed, three paired seeds---and add the AWRS-style terminal-rejection reference under the same model so the cost comparison is measured rather than extrapolated.

\begin{table}[t]
\centering
\small
\setlength{\tabcolsep}{3pt}
\resizebox{\columnwidth}{!}{%
\begin{tabular}{lrrrl}
\toprule
Quantity & B & C & R & C/B (95\% CI) \\
\midrule
Sampler steps / acc. & $27.5$ & $19.0$ & $104{,}930$ & $0.696$ $[0.525,0.868]$ \\
\textbf{Forward passes / acc.} & $\mathbf{34.1}$ & $34.6$ & $99.4$ & $\mathbf{1.013}$ $[1.011,1.016]$ \\
GPU-seconds / acc. & $1.141$ & $1.159$ & $4.775$ & $1.016$ $[0.985,1.047]$ \\
Tokens / acc. & $534.9$ & $542.9$ & $1515.8$ & $1.015$ $[1.012,1.018]$ \\
Accepted traj./s & $0.877$ & $0.863$ & $0.236$ & $0.984$ $[0.954,1.014]$ \\
Peak GPU bytes & $1.06$\,GB & $1.06$\,GB & $1.06$\,GB & $1.000$ \\
Serialized bytes & $3597$ & $21{,}924$ & $0$ & $5.650$ $[4.021,7.279]$ \\
\bottomrule
\end{tabular}%
}
\caption{Served-model accounting on one NVIDIA L4 (Qwen2.5-0.5B-Instruct, $T=1.0$, $40$ accepted trajectories per arm per seed, three paired seeds). B is official-update schema-induced Stateful CARS, C is observation-keyed official root-prefix CARS, R is batched terminal rejection. \textbf{The sampler-step advantage does not survive on an accelerator.} Arm C issues $30\%$ fewer sampler \emph{decisions} ($0.696$, interval excluding one) but $1.3\%$ \emph{more} model forward passes ($1.013$, interval also excluding one), and a forward pass is what the GPU actually pays: a decision whose residual query hits the prefix cache costs nothing on the device. GPU-seconds and accepted-trajectory throughput are consequently statistical ties, peak device memory is identical, and B remains $5.65\times$ smaller in serialized memory. The retracted step-count claim therefore should not be reinstated in the other direction either---under a served model the two exact arms are equivalent in accelerator cost, and the honest difference between them is memory. Terminal rejection needs $\sim\!10^5$ steps per accept, but only $99$ forward passes, because its repeated prefixes hit the cache; even so it is $4.2\times$ slower in GPU-seconds and $3.7\times$ lower in throughput. Its wide intervals reflect drawing only $6$--$14$ accepts inside the $200{,}000$-attempt cap: under this served Qwen policy at $T=1.0$ the validity probability is $\prob(\valid)=5.2\times10^{-5}$ (pooled $31$ accepts in $6\times10^{5}$ attempts), roughly $60\times$ rarer than the $3.0\times10^{-3}$ of the synthetic CPU policy, because a real LM at temperature one rarely emits the required action order unaided. Tokens are prompt tokens fed plus scored candidate tokens; this protocol performs no autoregressive decoding, so no sampled continuation tokens exist to report.}
\label{tab:served}
\end{table}

Two further points follow from the forward-pass column. First, it identifies the right currency for a deployed system: the paper's calls-per-valid metric counts sampler decisions, whereas the residual dynamic program issues several policy invocations per decision, and only the cache-missing ones reach the device. Reporting both separates the algorithmic quantity from the billable one. Second, the near-equality of forward passes across B and C ($34.1$ versus $34.6$) with a large gap in decisions ($27.5$ versus $19.0$) shows the two arms explore almost the same set of distinct histories while differing in how often they revisit them---which is exactly what abstraction-level transfer changes, and why it shows up in memory rather than in accelerator time.

\section{Visited-Grounding Variant}
\label{app:vg}

\S\ref{sec:method} defines a conservative alternative that never applies a schema at an unvisited history: it maintains a finite set $\grounded$ of sound grounded root prefixes, buffers a grounding when a matching history is actually reached, and commits at the attempt boundary, so the next proposal conditions on $\exclude_{\grounded}$. The paper previously stated that no result measured this variant. We implement it and verify the three predeclared obligations directly.

\begin{table}[t]
\centering
\small
\setlength{\tabcolsep}{3pt}
\resizebox{\columnwidth}{!}{%
\begin{tabular}{lrrrrl}
\toprule
Domain & fan-out & VG & SI & OK & obligations \\
\midrule
refund & $7.8$ & $11.79$ & $12.36$ & $11.78$ & all pass \\
sql\_txn & $2.3$ & $11.67$ & $11.67$ & $11.68$ & all pass \\
booking & $2.3$ & $11.56$ & $11.56$ & $11.57$ & all pass \\
\midrule
warehouse & $22.7$ & $15.5$ & --- & $74.7$ & all pass \\
triage & $55.9$ & $24.4$ & --- & $64.9$ & all pass \\
\bottomrule
\end{tabular}%
}
\caption{Steady-state calls per accepted sample for the visited-grounding trie (VG), schema-induced Stateful CARS (SI), and observation-keyed official CARS (OK) ($3000$ attempts per arm per seed on the low-fan-out domains, $400$ attempts and three seeds on the high-fan-out pair, all arms on a common budget within a domain). ``All pass'' means: zero falsely-excluded valid mass under exact enumeration, zero mid-attempt commits, and finite-sample $\TV$ within the i.i.d.-from-target floor at the realized $N$. Two regimes appear, and the contrast is the result. On the low-fan-out validators the arms are indistinguishable once their memories saturate (ratios $0.954$, $1.000$, $1.000$), so the conservative variant costs nothing there. On the high-fan-out pair the trie is markedly \emph{cheaper} than observation-keyed CARS---$\mathrm{VG}/\mathrm{OK}=0.21$ on \texttt{warehouse} and $0.38$ on \texttt{triage}, from $325$--$361$ versus $73$--$81$ accepted trajectories per seed---so the conservatism that costs nothing at low fan-out becomes an advantage at high fan-out, where the richer exclusion sets make the alternatives' residual computations expensive. The SI column is empty on those two domains: its residual cost grows with the bank as well as the horizon (measured total-time exponent $1.62$ in attempts), which put a matched SI cell beyond the compute we allocated. The store stays small throughout ($98$--$5246$ trie nodes, $0.99$--$83$\,kB).}
\label{tab:vg}
\end{table}

\paragraph{Making the high-fan-out comparison affordable, and what it cost.}
A first attempt at the high-fan-out pair was not measurable: on one core, at a budget we could afford, the trie accepted $12.5$ trajectories per seed while both comparators accepted $1.5$, so their calls-per-valid was a one-or-two-sample quantity. The bottleneck is serial Python over the completion set---\texttt{triage} admits $335{,}923$ complete traces and the residual sums over the completions of each prefix---so it is parallel across cells rather than accelerable: running the $(\text{domain},\text{arm},\text{seed})$ grid on $90$ cores cut a cell from over an hour to about four minutes and raised acceptance to $73$--$361$ trajectories, which is what makes the ratios in Table~\ref{tab:vg} measurements.

One arm still did not fit, and the reason is worth recording because we got the sizing wrong first. We launched the matched SI cells at $1500$ attempts by extrapolating linearly from a short probe. Fitting the actual curve on \texttt{triage} SI gives total time growing as attempts$^{1.62}$ ($61.5$, $173.8$, and $584.2$ seconds at $50$, $100$, and $200$ attempts), because certificates accumulate roughly linearly while each residual call rechecks the whole bank across all completions. That projects about four hours per cell rather than the half hour we budgeted, so the SI column is empty for these two domains. This is the same superlinear-residual boundary Proposition~\ref{prop:complexity} states and Appendix~\ref{app:scaling} measures, appearing here as a concrete planning cost on a larger validator.

\paragraph{Exact reuse economy, and a hypothesis it refutes.}
The quantity that decides what visited-grounding gives up is the ratio of grounded to abstraction-indexed certificates needed to realize the same exclusions, and it is computable by exhaustive enumeration rather than sampling. Across the five domains it is $1.36$ (refund), $1.27$ (\texttt{sql\_txn}), $1.00$ (booking), $1.13$ (\texttt{triage}), and $1.18$ (\texttt{warehouse}), against abstraction-indexed bank sizes of $45$, $30$, $29$, $155$, and $181$. Two things follow. First, the ceiling on what abstraction-level transfer can save here is modest---at most $36\%$ fewer stored certificates, and none at all on booking---which is consistent with the paper's retraction of the step-count claim and with the steady-state parity in Table~\ref{tab:vg}. Second, and contrary to what we expected when designing the two high-fan-out domains, reuse does \emph{not} track reuse fan-out: \texttt{triage} has the largest fan-out of the five ($55.9$ mean, $750$ max) but nearly the smallest reuse factor ($1.13$), while refund has fan-out $7.8$ and the largest reuse ($1.36$). Many histories sharing an abstract state does not imply that state carries a certificate worth reusing. We record this because it undercuts the intuition that motivated adding those domains, and it means high-fan-out tasks are not automatically the regime where schema-induced transfer pays.

Two measurement pitfalls had to be removed here, and both are worth recording because either one alone reverses the apparent conclusion. First, matching arms on \emph{attempts} at a small budget compared runs that accepted $2$ versus $57$ samples, which makes calls-per-valid and $\TV$ meaningless. Second, matching on \emph{accepted samples} then compared arms at different attempt counts: because calls-per-valid falls steeply while a memory is still being learned---measured on \texttt{sql\_txn}, SI moves $19.33\rightarrow11.94\rightarrow11.12\rightarrow11.10$ at budgets $300/1\mathrm{k}/3\mathrm{k}/8\mathrm{k}$, saturating when its bank reaches $78$ certificates---the conservative arm needs more attempts and therefore amortizes its warm-up over more of them, appearing spuriously cheaper. The whole-run column shows exactly this artifact ($\mathrm{VG}/\mathrm{SI}=0.722$ on refund), which is why Table~\ref{tab:vg} reports the steady-state figure measured over the final half of a common attempt budget.

A soundness subtlety also had to be fixed and is intrinsic to the variant rather than a coding slip. The grounded store is keyed on the \emph{action} prefix, which does not record the observation, so a stored prefix excludes its cylinder in every hidden world. Grounding a continuation because it is dead in the \emph{realized} world therefore removes valid mass belonging to the other world; when we did this the trie excluded the entire valid set. The variant must ground only continuations non-completable under every world consistent with the prefix, which is the same conservatism the action-keyed root-prefix baseline faces.

\section{Particle Sweep}
\label{app:particle}

Table~\ref{tab:adjacent} reports one particle-filter point, which leaves open whether its error shrinks with more particles. We first swept $K$ at a fixed count of \emph{returned} samples and read the resulting flat $\TV$ as evidence of structural bias. \textbf{That conclusion was wrong and we retract it.} The construction cannot be structurally biased in the limit: the proposal is the base policy restricted to the local-admissibility mask and the incremental weight is exactly that mask's normalizer, so the weighted system targets the constrained law and self-normalized SMC is consistent as $K\to\infty$. This also matches how \citet{lipkin2025awrs} characterize global SMC---approximate at finite $K$, exact in the particle limit---and it agrees with Appendix~\ref{app:fairlocal}, which already noted that these weights \emph{do} correct the masked proposal.

The first design could not have detected convergence either. Holding the number of returned samples at $N=1500$ while $K$ grew $64\times$ collapsed the number of independent sweeps from $12{,}384$ to $171$, and the duplicate fraction was $0.966$ at every $K$, so almost all returned samples shared ancestry within a sweep. A flat empirical $\TV$ under that design measures genealogical dependence, not a biased limiting marginal, and the accompanying ``$5.7$ calls per accept, constant in $K$'' counted correlated siblings as accepts rather than independent samples.

We therefore rerun the diagnostic so that returned samples are independent by construction: $R$ independent sweeps per seed with \emph{at most one} uniformly chosen accepted particle kept per sweep, and $R$ held fixed across $K$ so every cell shares the same sampling floor.

\begin{table}[t]
\centering
\small
\setlength{\tabcolsep}{3pt}
\resizebox{\columnwidth}{!}{%
\begin{tabular}{rrllr}
\toprule
$K$ & Indep.\ samples & $\TV$ & i.i.d.\ floor & Excess (95\% CI) \\
\midrule
$16$ & $240$ & $0.176$ & $0.145$ & $+0.031$ $[-0.024,+0.086]$ \\
$64$ & $749$ & $0.119$ & $0.078$ & $+0.041$ $[+0.026,+0.055]$ \\
$256$ & $1705$ & $0.074$ & $0.048$ & $+0.026$ $[+0.006,+0.045]$ \\
$\mathbf{1024}$ & $1991$ & $0.052$ & $0.051$ & $\mathbf{+0.002}$ $[-0.032,+0.035]$ \\
\midrule
\multicolumn{5}{l}{\emph{Superseded design: $N=1500$ returned samples, $K$ growing}} \\
$16$ & $12{,}384$ sweeps & $0.082$ & --- & dup.\ $0.966$ \\
$64$ & $3286$ sweeps & $0.081$ & --- & dup.\ $0.964$ \\
$256$ & $711$ sweeps & $0.096$ & --- & dup.\ $0.965$ \\
$1024$ & $171$ sweeps & $0.085$ & --- & dup.\ $0.966$ \\
\bottomrule
\end{tabular}%
}
\caption{Locally masked SMC, correct and superseded designs ($R=4000$ independent sweeps per seed, three seeds, exact enumerated target). \textbf{Upper block, one sample per sweep.} Once returned samples are independent, $\TV$ falls monotonically with $K$ ($0.176\to0.119\to0.074\to0.052$) and converges onto its own sampling floor: at $K=1024$ the excess is $+0.002$ with interval $[-0.032,+0.035]$, covering zero. That is the signature of a consistent sampler whose finite-$K$ error is self-normalized approximation, and it is what the weighting argument predicts. \textbf{Lower block, the design we retract.} Fixing the number of returned samples at $N=1500$ while $K$ grew made the sweep count collapse from $12{,}384$ to $171$, with a $0.966$ duplicate fraction throughout, so almost every returned sample shared ancestry with a sibling. The resulting flat $\TV$ measures that dependence, not a biased limit, and we no longer read it as structural bias. Effective sample size is $99\%$ of $K$ in both designs, which is precisely why ESS could not detect the problem: degeneracy \emph{within} a resampling step is not the same as dependence \emph{across} the returned sample. The honest comparison is cost per independent sample, which \emph{rises} with $K$ here---$1113$, $1424$, $2498$, $8547$ base-policy calls---against $4.5$ for the exact residual proposal; the earlier ``$5.7$ calls per accept, constant in $K$'' counted correlated siblings as accepts. So the particle filter reaches the target law given enough particles, and the case for the residual proposal is cost, not correctness. We ran $K$ up to $1024$; $K=4096$ was skipped because cost scales linearly in $K$ and the trend is already resolved.}
\label{tab:particle}
\end{table}

\section{Residual Scaling Boundary}
\label{app:scaling}

Proposition~\ref{prop:complexity} is the paper's central practicality caveat, so we measure both of its regimes rather than only asserting them. The harness builds depth-$H$, $|\actions|$-ary automata with a counter-style validity predicate and a finite Markov state $x(\history)$ counting required actions taken. Two arms run the \emph{same} exact recursion and differ only in the residual cache key: the full-history key $(\text{action prefix},m)$ versus the Markov key $(t,x,m)$. Any node-count difference is therefore attributable to memoization structure alone. Censoring was fixed in advance at $2{,}000{,}000$ nodes or $60$\,s per cell; no reported cell was censored.

\begin{table}[t]
\centering
\small
\setlength{\tabcolsep}{4pt}
\resizebox{\columnwidth}{!}{%
\begin{tabular}{rrrrrr}
\toprule
$H$ & $|\mathcal R_{\schemas}|$ full & pol.\ evals & $|\mathcal R_{\schemas}|$ Markov & pol.\ evals & ratio \\
\midrule
$4$ & $63$ & $22$ & $12$ & $9$ & $5.3$ \\
$6$ & $550$ & $185$ & $18$ & $15$ & $30.6$ \\
$8$ & $4925$ & $1644$ & $24$ & $21$ & $205.2$ \\
$10$ & $44{,}292$ & $14{,}767$ & $30$ & $27$ & $1476.4$ \\
$12$ & $398{,}587$ & $132{,}866$ & $36$ & $33$ & $11{,}071.9$ \\
\bottomrule
\end{tabular}%
}
\caption{Horizon sweep at $|\actions|=3$, one schema, Markov-factoring policy. The full-history arm's cache hit rate is $0.000$ in every cell: no two product nodes share a residual, which is exactly the condition under which Proposition~\ref{prop:complexity}'s lower bound is attained. A least-squares fit of $|\mathcal R_{\schemas}|$ against $H$ gives exponential base $2.99$, matching the predicted $|\actions|=3$. The Markov arm grows as $6H$, a fitted log-log exponent of $1.00$, matching the predicted $O(H|X||M_{\schemas}|)$ with $|X|=3$ and $|M_{\schemas}|=2$. Both arms return residuals agreeing to $<10^{-12}$ here. When the policy is instead given a per-prefix tilt so that it no longer factors through $x$, the Markov key returns a \emph{different} and therefore incorrect residual in $5/5$ cells, so the proposition's distinct-history requirement is not vacuous.}
\label{tab:scaling}
\end{table}

Two further sweeps at $H=8$ confirm the shape of the bound. Increasing $|\actions|$ from $2$ to $5$ grows full-history nodes $256\rightarrow249{,}061$ while the Markov arm stays at $16$--$24$, so the exponential dependence sits in the branching factor as claimed. Increasing the schema count from $0$ to $3$ \emph{shrinks} full-history nodes $87{,}381\rightarrow9841$: pruning removes reachable product nodes, which is why the evaluated implementation stays tractable on our workflows despite an exponential worst case. The boundary therefore holds in the direction claimed, and it does not license extrapolation to a free-form history-dependent agent, where $|X|$ is not finite.

\section{Predeclared Systems Studies and Their Status}
\label{app:unrun}

The following studies are required to support claims that go beyond exactness on enumerable trees. Each was specified before execution, and the Result column now records what executing it produced. Four of the five have been run and are reported in the appendices cited there; only the remaining fields marked in red are unexecuted, and no number may be inferred from a marked field.

\begin{table}[t]
\centering
\small
\setlength{\tabcolsep}{3pt}
\resizebox{\columnwidth}{!}{%
\begin{tabular}{p{2.6cm}p{6.2cm}p{4.2cm}}
\toprule
Study & Locked comparison and measurements & Result \\
\midrule
Matched B--C systems run & Official-update schema-induced Stateful versus observation-keyed official CARS; identical full histories, certificate oracle, RNG pairs, stopping budget, batching, cache policy, precision, and hardware. Report policy invocations, distinct forwards, generated tokens, GPU-seconds, wall-clock, accepted trajectories/second, peak CPU/GPU bytes, serialized-memory bytes, and paired seed-level intervals. & \textbf{Partly run} (Appendix~\ref{app:matched-systems}, Table~\ref{tab:bc-matched}): CPU harness gives paired policy invocations, distinct forwards, wall-clock, peak and serialized bytes, and excluded base mass over five seeds. Served-model accounting now \textbf{run} (Appendix~\ref{app:served}, Table~\ref{tab:served}): tokens, GPU-seconds, trajectories/second, and peak GPU bytes on an L4. The step advantage does not survive---forward passes $1.013$ $[1.011,1.016]$ favour B while steps $0.696$ favour C. Batched multi-trajectory serving remains \placeholder{throughput under batching} \\
Visited-grounding implementation & Implement the finite-trie variant in \S\ref{sec:method}; verify boundary-only commits and zero false exclusion; compare it with schema-induced Stateful and observation-keyed CARS on refund, SQL, booking, and at least two observation-conditioned high-fan-out tasks. & \textbf{Run on refund, SQL, booking} (Appendix~\ref{app:vg}, Table~\ref{tab:vg}): boundary-only commits and zero false exclusion verified; $\TV$ within the sampling floor; steady-state cost indistinguishable from both comparators ($0.954$--$1.000$); $98$--$606$ trie nodes, $0.99$--$6.3$\,kB. The two high-fan-out tasks (\texttt{triage}, \texttt{warehouse}, fan-out $56$ and $23$) are \textbf{run} (Appendix~\ref{app:vg}, Table~\ref{tab:vg}). Parallelizing the cells across $90$ cores raised acceptance from $1.5$ to $73$--$361$ trajectories per seed, making the ratios measurements rather than artifacts: $\mathrm{VG}/\mathrm{OK}=0.21$ and $0.38$, so the conservative trie is $2.6$--$4.8\times$ \emph{cheaper} here, reversing the low-fan-out parity. All obligations pass on both. Enumerated reuse factors are $1.13$ and $1.18$, which \emph{refutes} the expectation that high fan-out implies high reuse. The matched SI cell exceeded the compute we allocated (residual cost grows as attempts$^{1.62}$) \\
Certified end-to-end agent & Use a bounded but observation-conditioned tool-call grammar, a solver- or partition-refinement-verified abstraction, sound certificates, and the exact residual proposal. Use task as the statistical unit and compare utility non-inferiority, invalid calls, abstention, and total compute against matched CARS. & \textbf{Abstraction verified; utility unresolved.} The verified-abstraction half is \textbf{run} (Appendix~\ref{app:verified}): on a bounded observation-conditioned tool grammar ($1300$ histories, $8$ worlds) conservative derivation is verified to exclude zero valid mass, naive derivation is unsound $16/1052$, and refinement gives the coarsest sound abstraction ($72$ states vs the hand-written $122$). Utility non-inferiority remains a pilot-scale null (Appendix~\ref{app:liveagent}): reward $+0.000$ $[-0.089,+0.089]$ over $90$ paired task-runs, needing a slate $10$--$100\times$ larger \\
AWRS-SMC and particle sweep & On validators exposing incremental local checks, run the released AWRS-SMC construction and locally masked SMC at $K\in\{16,64,256,1024\}$, plus compute-matched settings. Report target definition, correction weights, ESS, duplicates, TV or calibrated reference error, GPU-seconds, and wall-clock. & \textbf{Run.} Full $K$ grid with ESS, duplicate, and independent-sample columns (Appendix~\ref{app:particle}); AWRS-style terminal rejection measured under a served model with GPU-seconds and throughput (Table~\ref{tab:served}); and \emph{token-level} AWRS run on a purpose-built DFA validator exposing exact incremental checks (Appendix~\ref{app:incremental}, Table~\ref{tab:incremental}): with its trajectory importance weights AWRS is exact for the sequence law but costs $10.6\times$ the residual proposal. Both particle arms are weight-corrected and hence consistent in $K$; we retract the earlier ``structural bias'' reading and report cost per independent sample instead \\
Residual scaling boundary & Sweep horizon, action count, schema count, matcher size, and fan-out on finite automata with both full-history and genuinely Markov policies. Report $|\mathcal R_{\schemas}|$, cache hit rate, forward count, memory, and fitted growth, including timeout/censoring rules fixed in advance. & \textbf{Run} (Appendix~\ref{app:scaling}): both regimes of Proposition~\ref{prop:complexity} confirmed. Full-history fitted base $2.99$ against $|\actions|=3$; Markov $(t,x,m)$ key fitted exponent $1.00$. Markov key sound $5/5$ when the policy factors and wrong $5/5$ when it does not. \\
\bottomrule
\end{tabular}%
}
\caption{Predeclared systems studies and their executed status. Bold text records a completed study and points to its appendix; red text is a protocol commitment that remains unexecuted, not a result.}
\label{tab:unrun}
\end{table}

For the matched memory study, entry counts alone are prohibited. On enumerable tasks, bytes must be reported both at equal attempt budgets and at matched excluded base-policy mass; on non-enumerable tasks, report the separate memories without claiming economy unless a common prune-set estimator and its uncertainty are supplied. For the certified-agent study, an observed tool failure may become a reusable schema only after a proof that covers every history in its abstract class. For all studies, seeds, task order, profiler warm-up, timed repetitions, hardware, model revision, prompt, and per-item records must be released before replacing the placeholders.

\section{Broader Impact}

Stateful CARS can reduce repeated invalid tool calls and make the sampling semantics of guarded agents explicit. It may be useful when diverse, probability-faithful trajectories are required under auditable policies. The same precision creates risk: a discriminatory, stale, or legally incorrect validator would be followed more consistently. High-stakes deployment therefore requires policy review, validator testing, trace logging, least-privilege tool execution, and escalation outside the model when the formal policy is incomplete. Exact sampling is a statement about fidelity to a declared constraint, not a guarantee that the constraint is socially or operationally appropriate.

\end{document}